\documentclass[11pt]{article}
\usepackage{amsmath,amssymb,amscd,amsthm}
\usepackage[numbers]{natbib}                                 
\usepackage{graphicx}
\usepackage[pdftex,colorlinks]{hyperref}
\usepackage[margin=1.0in]{geometry}

\usepackage{algorithm}
\usepackage{algorithmic}

\newtheorem{theorem}{Theorem}[section]
\newtheorem{lemma}[theorem]{Lemma}
\newtheorem{remark}[theorem]{Remark}

\newtheorem{counter-example}[theorem]{Counter example}

\newtheorem{open question}[theorem]{Open question}

\newtheorem{definition}[theorem]{Definition}

\newcommand{\cb}{{\cal B}}
\newcommand{\cd}{{\cal D}}

\newcommand{\cc}{{\cal C}}

\newcommand{\cq}{{\cal Q}}

\newcommand{\ct}{{\cal T}}

\newcommand{\ch}{{\cal H}}

\newcommand{\cw}{{\cal W}}
\newcommand{\cl}{{\cal L}}
\newcommand{\cf}{{\cal F}}

\newcommand{\cx}{{\cal X}}
\newcommand{\cy}{{\cal Y}}

\newcommand{\cs}{{\cal S}}
\newcommand{\cn}{{\cal N}}

\newcommand{\x}{\mathbf{x}}
\newcommand{\f}{\mathbf{f}}
\newcommand{\y}{\mathbf{y}}
\newcommand{\z}{\mathbf{z}}
\newcommand{\bt}{\mathbf{t}}
\newcommand{\bw}{\mathbf{w}}
\newcommand{\bv}{\mathbf{v}}
\newcommand{\ba}{\mathbf{a}}
\newcommand{\bu}{\mathbf{u}}
\newcommand{\be}{\mathbf{e}}
\newcommand{\bb}{\mathbf{b}}

\DeclareMathOperator*{\sign}{sign}

\newcommand{\ball}{\mathbb{B}}

\newcommand{\reals}{{\mathbb R}}
\newcommand{\sphere}{{\mathbb S}}

\newcommand{\var}{\mathrm{VAR}}

\newcommand{\fat}{\mathrm{Fat}}

\newcommand{\med}{\mathrm{med}}
\newcommand{\rep}{\mathrm{rep}}
\newcommand{\len}{\mathrm{len}}

\newcommand{\tr}{\mathrm{tr}}

\DeclareMathOperator*{\E}{\mathbb{E}}

\newcommand{\inner}[1]{\left\langle #1 \right\rangle}
\newcommand{\todo}[1]{{\bf TODO: #1}}

\title{Generalization Bounds for Neural Networks via Approximate Description Length}

\author{
Amit Daniely\thanks{School of Computer Science and Engineering, The Hebrew University, Jerusalem, and Google Research, Tel-Aviv.} \hspace{1cm}
Elad Granot\thanks{School of Computer Science and Engineering, The Hebrew University, Jerusalem} 
}

\begin{document}
\maketitle

\thispagestyle{empty}

\begin{abstract}
We investigate the sample complexity of networks with bounds on the magnitude of its weights. 
In particular, we consider the class
\[
\cn = \left\{W_t\circ\rho\circ W_{t-1}\circ\rho\ldots\circ \rho\circ W_{1} : W_1,\ldots,W_{t-1}\in M_{d\times d}, W_t\in M_{1,d}  \right\}
\]
where the spectral norm of each $W_i$ is bounded by $O(1)$, the Frobenius norm is bounded by $R$, and $\rho$ is the sigmoid function $\frac{e^x}{1 + e^x}$ or the smoothened ReLU function $ \ln\left(1 + e^x\right)$.
We show that for any depth $t$, if the inputs are in $[-1,1]^d$, the sample complexity of $\cn$ is $\tilde O\left(\frac{dR^2}{\epsilon^2}\right)$. This bound is optimal up to log-factors, and substantially improves over the previous state of the art of $\tilde O\left(\frac{d^2R^2}{\epsilon^2}\right)$, that was established in a recent line of work~\cite{neyshabur2015norm, bartlett2017spectrally, neyshabur2017pac, golowich2017size, arora2018stronger, neyshabur2018role}.

We furthermore show that this bound remains valid if instead of considering the magnitude of the $W_i$'s, we consider the magnitude of $W_i - W_i^0$, where $W_i^0$ are some reference matrices, with spectral norm of $O(1)$. By taking the $W_i^0$ to be the matrices at the onset of the training process, we get sample complexity bounds that are sub-linear in the number of parameters, in many {\em typical} regimes of parameters.  

To establish our results we develop a new technique to analyze the sample complexity of families $\ch$ of predictors. 
We start by defining a new notion of a randomized approximate description of functions $f:\cx\to\reals^d$. We then show that if there is a way to approximately describe functions in a class $\ch$ using $d$ bits, then $\frac{d}{\epsilon^2}$ examples suffices to guarantee uniform convergence. Namely, that the empirical loss of all the functions in the class is $\epsilon$-close to the true loss. Finally, we develop a set of tools for calculating the approximate description length of classes of functions that can be presented as a composition of linear function classes and non-linear functions.

\end{abstract}

\newpage
\tableofcontents

\thispagestyle{empty}

\newpage

\setcounter{page}{1}

\section{Introduction}

We analyze the sample complexity of networks with bounds on the magnitude of their weights. Let us consider a prototypical case, where the input space is $\cx = [-1,1]^d$, the output space is $\reals$, the number of layers is $t$, all hidden layers has $d$ neurons, and the activation function is $\rho:\reals\to\reals$.
The class of functions computed by such an architecture is
\[
\cn = \left\{W_t\circ\rho\circ W_{t-1}\circ\rho\ldots\circ \rho\circ W_{1} : W_1,\ldots,W_{t-1}\in M_{d\times d}, W_t\in M_{1,d}  \right\}
\]
As the class $\cn$ is defined by $(t-1)d^2 + d = O(d^2)$ parameters, classical results (e.g. \cite{AnthonyBa99}) tell us that order of $d^2$ examples are sufficient and necessary in order to learn a function from $\cn$ (in a standard worst case analysis). However, modern networks often succeed to learn with substantially less examples. One way to provide alternative results, and a potential  explanation to the phenomena, is to take into account the magnitude of the weights. This approach was a success story in the days of SVM~\cite{BartlettMe02} and Boosting~\cite{SchapireFrBaLe97}, provided a nice explanation to generalization with sub-linear (in the number of parameters) number of examples, and was even the deriving force behind algorithmic progress. It seems just natural to adopt this approach in the context of modern networks. For instance, it is natural to consider the class
\[
\cn_{R} = \left\{W_t\circ\rho\circ W_{t-1}\circ\rho\ldots\circ \rho\circ W_{1} :  \forall i, \|W_i\|_F\le R,  \|W_i\|\le O(1)\right\}
\]
where $\|W\| = \max_{\|\x\|=1}\|W\x\|$ is the spectral norm and $\|W\|_F=\sqrt{\sum_{i,j=1}^dW_{ij}^2}$ is the Frobenius norm. 
This class has been analyzed in several recent works~\cite{neyshabur2015norm, bartlett2017spectrally, neyshabur2017pac, golowich2017size, arora2018stronger, neyshabur2018role}. Best known results show a sample complexity of $\tilde{O}\left( \frac{d^2R^2}{\epsilon^2} \right)$ (for the sake of simplicity, in the introduction, we ignore the dependence on the depth in the big-O notation). In this paper we prove, for various activations, a stronger bound of $\tilde{O}\left( \frac{dR^2}{\epsilon^2} \right)$, which is optimal, up to log factors, for constant depth networks.

How good is this bound? Does it finally provide sub-linear bound in typical regimes of the parameters? To answer this question, we need to ask how large $R$ is. While this question of course don't have a definite answer, empirical studies (e.g. \cite{sutskever2013importance}) show that it is usually the case that the norm (spectral, Frobenius, and others) of the weight matrices is at the same order of magnitude as the norm of the matrix in the onset of the training process. In most standard training methods, the initial matrices are random matrices with independent (or almost independent) entries, with mean zero and variance of order $\frac{1}{d}$. The Frobenius norm of such a matrix is of order $\sqrt{d}$. Hence, the magnitude of $R$ is of order $\sqrt{d}$. Going back to our $\tilde{O}\left( \frac{dR^2}{\epsilon^2} \right)$ bound, we get a sample complexity of  $\tilde{O}\left( \frac{d^2}{\epsilon^2} \right)$, which is unfortunately still linear in the number of parameters. 

Since our bound is almost optimal, we can ask whether this is the end of the story? Should we abandon the aforementioned approach to network sample complexity? A more refined examination of the training process suggests another hope for this approach. Indeed, the training process doesn't start from the zero matrix, but rather form a random initialization matrix. Thus, it stands to reason that instead of considering the magnitude of the weight matrices $W_i$, we should consider the magnitude of $W_i - W_i^0$, where $W_i^0$ is the initial weight matrix. Indeed, empirical studies~\cite{nagarajan2019generalization} show that the Frobenius norm of $W_i - W_i^0$ is often order of magnitude smaller than the Frobenius norm of $W_i$. Following this perspective, it is natural to consider the class
\[
\cn_{R}(W^0_1,\ldots,W^0_t) = \left\{W_t\circ\rho\circ W_{t-1}\circ\rho\ldots\circ \rho\circ W_{1} : \| W_i - W_i^0 \|\le O(1) , \| W_i - W_i^0 \|_F \le R \right\}
\]
For some fixed matrices, $W^0_1,\ldots,W^0_t$ of spectral norm\footnote{The bound of $O(1)$ on the spectral norm of the $W^0_i$'s and $W_i - W_i^0$ is again motivated by the practice of neural networks -- the spectral norm of $W_i^0$, with standard initializations, is $O(1)$, and empirical studies~\cite{nagarajan2019generalization, sutskever2013importance} show that the spectral norm of $W_i - W_i^0$ is usually very small.} $O(1)$. It is natural to expect that considering balls around the initial $W_i^0$'s instead of zero, shouldn't change the sample complexity of the class at hand. In other words, we can expect that the sample complexity of $\cn_{R}(W^0_1,\ldots,W^0_t)$  is approximately $\tilde{O}\left( \frac{dR^2}{\epsilon^2} \right)$ -- the sample complexity of $\cn_R$. Such a bound would finally be sub-linear, as in practice, it is often the case that $R^2\ll d$.

This approach was pioneered by  \citet{bartlett2017spectrally} who considered the class
\[
\cn^{2,1}_{R}(W^0_1,\ldots,W^0_t) = \left\{W_t\circ\rho\circ W_{t-1}\circ\rho\ldots\circ \rho\circ W_{1} : \| W_i - W_i^0 \|\le O(1) , \| W_i - W_i^0 \|_{2,1} \le R \right\}
\]
where $\|W\|_{2,1} = \sum_{i=1}^d\sqrt{\sum_{j=1}^dW_{ij}^2}$. For this class they proved a sample complexity bound of $\tilde{O}\left( \frac{dR^2}{\epsilon^2} \right)$. Since, $\|W\|_{2,1} \le \sqrt{d}\|W\|_{F}$, this implies a sample complexity bound of $\tilde{O}\left( \frac{d^2R^2}{\epsilon^2} \right)$ on $\cn_{R}(W^0_1,\ldots,W^0_t)$, which is still not sublinear. We note that  $\|W\|_{2,1} = \Theta(\sqrt{d})$ even if $W$ is a random matrix with variance that is calibrated so that $\|W\|_{F} = \Theta(1)$ (namely, each entry has variance $\frac{1}{d^2}$).
In this paper we finally prove a sub-linear sample complexity bound of $\tilde{O}\left( \frac{dR^2}{\epsilon^2} \right)$ on $\cn_{R}(W^0_1,\ldots,W^0_t)$.

To prove our results, we develop a new technique for bounding the sample complexity of function classes. Roughly speaking, we define a notion of approximate description of a function, and count how many bits are required in order to give an approximate description for the functions in the class under study. We then show that this number, called the {\em approximate description length (ADL)}, gives an upper bound on the sample complexity. 
The advantage of our method over existing techniques is that it behaves nicely with compositions. That is, once we know the approximate description length of a class $\ch$ of functions from $\cx$ to $\reals^d$, we can also bound the ADL of $\rho\circ\ch$, as well as $\cl\circ\ch$, where $\cl$ is a class of linear functions. This allows us to utilize the compositional structure of neural networks.

\section{Preliminaries}

\subsection{Notation}
We denote by $\med(x_1,\ldots,x_k)$ the median of $x_1,\ldots,x_k\in\reals$. 
For vectors $\x^1,\ldots,\x^k\in\reals^d$ we denote $\med(\x^1,\ldots,\x^k) = \left(\med(x_1^1,\ldots,x_1^k),\ldots,\med(x_d^1,\ldots,x_d^k)\right)$.
We use $\log$ to denote $\log_2$, and $\ln$ to denote $\log_e$
An expression of the form $f(n)\lesssim g(n)$ means that there is a universal constant $c>0$ for which $f(n)\le c g(n)$.
For a finite set $A$ and $f:A\to\reals$ we let $\E_{x\in A}f = \E_{x\in A}f(a) = \frac{1}{|A|}\sum_{a\in A}f(a)$.
We denote $\ball_M^d = \{\x\in\reals^d : \|\x\| \le M\}$ and $\ball^d = \ball_1^d$. Likewise, we denote  $\sphere^{d-1} = \{\x\in\reals^d : \|\x\| =1\}$.
We denote the Frobenius norm of a matrix $W$ by $\|W\|_F^2  = \inner{W,W} =  \sum_{ij}W_{ij}^2$, while the spectral norm is denoted by $\|W\| = \max_{\|\x\| = 1}\|W\x\|$.
For a pair of vectors $\x,\y\in\reals^d$ we denote by $\x\y\in\reals^d$ their point-wise product $\x\y = (x_1y_1,\ldots,x_dy_d)$. For a scalar $a$ we denote by $\vec{a}\in\reals^d$ the vector whose all coordinates are $a$. 
Let $V$ be a finite dimensional inner product space.
A {\em standard Gaussian} in $V$ is a centered Gaussian vector $X\in V$ such that $\var(\inner{\bu,X})=\|\bu\|^2$ for any $\bu\in V$. For a subspace $U\subset V$ we denote by $P_U$ the orthogonal projection on $U$.

\subsection{Uniform Convergence and Covering Numbers}
Fix an instance space $\cx$, a label space $\cy$ and a loss $\ell:\reals^d\times\cy\to [0,\infty)$. We say that $\ell$ is Lipschitz / Bounded / etc. if for any $y\in\cy$, $\ell(\cdot,y)$ is. Fix a class $\ch$ from $\cx$ to $\reals^d$.
For a distribution $\cd$ and a sample $S\in\left(\cx\times\cy\right)^m$ we define the {\em representativeness} of $S$ as
\[
\rep_\cd(S,\ch) = \sup_{h\in\ch}\ell_{\cd}(h) - \ell_S(h)\text{ where }\ell_\cd(h)=\E_{(x,y)\sim\cd}\ell(h(x),y)\text{ and }\ell_S(h)=\frac{1}{m}\sum_{i=1}^m\ell(h(x_i),y_i)
\]
We note that if $\rep_\cd(S,\ch)\le \epsilon$ then any algorithm that is guaranteed to return a function $\hat h\in\ch$ will enjoy a generalization bound $\ell_\cd(h) \le \ell_S(h) + \epsilon$. In particular, the ERM algorithm will return a function whose loss is optimal, up to an additive factor of $\epsilon$.
We will focus on  bounds on $\rep_\cd(S,\ch)$ when $S\sim\cd^m$. To this end, we will rely on the connection between representativeness and the {\em covering numbers} of $\ch$.

\begin{definition}
Fix a class $\ch$ of functions from $\cx$ to $\reals^d$, an integer $m$, $\epsilon>0$ and $1\le p\le \infty$.
We define $N_p(\ch,m,\epsilon)$ as the minimal integer for which the following holds. For every $A\subset \cx$ of size $\le m$ there exists $\tilde \ch \subset \left(\reals^d\right)^\cx$ such that $\left|\tilde \ch\right|\le N_p(\ch,m,\epsilon)$ and for any $h\in\ch$ there is $\tilde h\in\tilde\ch$ with $\left(\E_{x\in A} \left\|h(x) - \tilde h(x)\right\|_\infty^p\right)^\frac{1}{p}\le \epsilon$. For $p=2$, we denote $N(\ch,m,\epsilon) = N_2(\ch,m,\epsilon)$
\end{definition}

\begin{lemma}\cite{shalev2014understanding}
Let $\ell:\reals^d\times\cy\to \reals$ be $B$-bounded. Then
\[
\E_{S\sim\cd^m} \rep_\cd(S,\ch) \le B2^{-M+1} + \frac{12B}{\sqrt{m}}\sum_{k=1}^M 2^{-k}\sqrt{\ln\left(  N(\ell\circ \ch,m,B2^{-k}) \right)}
\]
Furthermore, with probability at least $1-\delta$,
\[
 \rep_\cd(S,\ch) \le B2^{-M+1} + \frac{12B}{\sqrt{m}}\sum_{k=1}^M 2^{-k}\sqrt{\ln\left(  N(\ell\circ \ch,m,B2^{-k}) \right)} + B\sqrt{\frac{2\ln\left(2/\delta\right)}{m}}
\]
\end{lemma}
We conclude with a special case of the above lemma, which will be useful in this paper.
\begin{lemma}\label{lem:cover_to_gen}
Let $\ell:\reals^d\times\cy\to \reals$ be $L$-Lipschitz w.r.t. $\|\cdot\|_\infty$ and $B$-bounded. Assume that for any $0<\epsilon\le 1$, $\log\left(N(\ch,m,\epsilon)\right) \le \frac{n}{\epsilon^2} $.
Then
\[
\E_{S\sim\cd^m}\rep_\cd(S,\ch) \lesssim   \frac{(L+B)\sqrt{n} }{\sqrt{m}} \log(m)
\]
Furthermore, with probability at least $1-\delta$,
\[
\rep_\cd(S,\ch) \lesssim   \frac{(L+B)\sqrt{n} }{\sqrt{m}} \log(m) + B\sqrt{\frac{2\ln\left(2/\delta\right)}{m}}
\]

\end{lemma}
\begin{proof}
Denote 
\[
A = B2^{-M+1} + \frac{12B}{\sqrt{m}}\sum_{k=1}^M 2^{-k}\sqrt{\ln\left(  N(\ell\circ \ch,m,B2^{-k}) \right)}
\]
We will show that $A \lesssim \frac{(L+B)\sqrt{n} }{\sqrt{m}} \log(m)$.
We have, if $\frac{B2^{-k}}{L} \le 1$,
\[
\ln\left(N(\ell\circ \ch,m,B2^{-k})\right) \le \ln\left(N\left(\ch,m,\frac{B}{L}2^{-k}\right)\right) \le \frac{n L^22^{2k}}{B^2} + n
\]
Hence,
\[
A \le B2^{-M+1} + \frac{12B}{\sqrt{m}}\sum_{k=1}^M\frac{\sqrt{n} L}{B} 
+ \frac{12B}{\sqrt{m}}\sum_{k=1}^M2^{-k}\sqrt{n}
\le B2^{-M+1} + \frac{12(LM+B)\sqrt{n}}{\sqrt{m}} 
\]
Choosing $M=\log\left(\sqrt{\frac{m}{n}}\right)$ we get,
\[
A \le   \frac{12(L\log\left(\sqrt{\frac{m}{n}}\right)+B)\sqrt{n} + B\sqrt{n}}{\sqrt{m}} 
\]
\end{proof}

\subsection{Basic Inequalities}
\begin{lemma}\label{lem:median}
Let $X_1,\ldots,X_n$ be independent r.v. with that that are $\sigma$-estimators to $\mu$. Then
\[
\Pr\left(|\med(X_1,\ldots,X_n)-\mu|>k\sigma\right) <  \left(\frac{2}{k}\right)^n
\]
\end{lemma}
\begin{proof}
We have that $\Pr(|X_i-\mu|>k\sigma)\le  \frac{1}{k^2}$. It follows that the probability that $\ge \frac{n}{2}$ of $X_1,\ldots,X_n$ fall outside of the segment $\left(\mu - k\sigma,\mu + k\sigma\right)$ is bounded by 
\[
\binom{n}{\left\lceil n/2\right\rceil}\left(\frac{1}{k^2}\right)^{\left\lceil n/2\right\rceil} < 2^n\left(\frac{1}{k^2}\right)^{\left\lceil n/2\right\rceil} \le \left(\frac{2}{k}\right)^n
\]
\end{proof}
\section{Simplified Approximate Description Length}\label{sec:adl_intro}
To give a soft introduction to our techniques, we first consider a simplified version of it.
We next define the {\em approximate description length} of a class $\ch$ of functions from $\cx$ to $\reals^d$, which quantifies the number of bits it takes  to  approximately describe a function from $\ch$. We will use the following notion of approximation

\begin{definition}
A random vector $X\in \reals^d$ is a $\sigma$-estimator to $\x\in\reals^d$ if
\[
\E X = \x\text{ and }\forall \bu\in\sphere^{d-1},\;\var(\inner{\bu,X}) = \E \inner{\bu, X- \x}^2 \le \sigma^2
\]
A random function $\hat f:\cx\to \reals^d$ is a $\sigma$-estimator to $f:\cx\to \reals^d$ if for any $x\in\cx$, $\hat f(x)$ is a $\sigma$-estimator to $f(x)$.
\end{definition}

A {\em $(\sigma, n)$-compressor} $\cc$ for a class $\ch$ takes as input a function $h\in \ch$, and outputs a (random) function $\cc h$ such that (i) $\cc h$ is a $\sigma$-estimator of $h$ and (ii) it takes $n$ bits to describe $\cc h$. Formally,

\begin{definition}
A {\em $(\sigma, n)$-compressor} for $\ch$ is a triplet $(\cc,\Omega,\mu)$ where $\mu$ is a probability measure on  $\Omega$, and $\cc$ is a function $\cc:\Omega\times \ch \to \left(\reals^d\right)^\cx$ such that 
\begin{enumerate}
\item
For any $h\in\ch$ and $x\in\cx$, $(\cc_\omega h)(x),\;\omega\sim\mu$ is a $\sigma$-estimator of $h(x)$.
\item
There are functions $E:\Omega\times \ch \to \{\pm 1\}^n$ and $D:\{\pm 1\}^n\to \left(\reals^d\right)^\cx$ for which $\cc = D\circ E$
\end{enumerate}
\end{definition}

\begin{definition}
We say that a class $\ch$ of functions from $\cx$ to $\reals^d$ has {\em approximate description length} $n$ if there exists a $(1,n)$-compressor for $\ch$ 
\end{definition}

It is not hard to see that if $(\cc,\Omega,\mu)$ is a $(\sigma, n)$-compressor for $\ch$, then 
\[
(\cc_{\omega_1,\ldots,\omega_k}h)(x):=\frac{\sum_{i=1}^k (\cc_{\omega_i}h)(x)  }{k}
\]
is a $\left(\frac{\sigma}{\sqrt{k}}  , kn\right)$-compressor for $\ch$. Hence, if the approximate description length of $\ch$ is $n$, then for any $1\ge \epsilon>0$ there exists an $\left(\epsilon, n\lceil \epsilon^{-2} \rceil\right)$-compressor for $\ch$.

We next connect the approximate description length, to covering numbers and representativeness. We separate it into two lemmas, one for $d=1$ and one for general $d$, as for $d=1$ we can prove a slightly stronger bound.

\begin{lemma}\label{lem:adl_to_cov_1d_intro}
Fix a class $\ch$ of functions from $\cx$ to $\reals$  with approximate description length $n$. Then,
\[
\log\left(N(\ch,m,\epsilon)\right) \le {n\left\lceil \epsilon^{-2} \right\rceil}
\]
Hence, if $\ell:\reals^d\times\cy\to \reals$ is $L$-Lipschitz and $B$-bounded, then for any distribution $\cd$ on $\cx\times\cy$
\[
\E_{S\sim\cd^m}\rep_\cd(S,\ch) \lesssim   \frac{(L+B)\sqrt{n} }{\sqrt{m}} \log(m)
\]
Furthermore, with probability at least $1-\delta$,
\[
\rep_\cd(S,\ch) \lesssim   \frac{(L+B)\sqrt{n} }{\sqrt{m}} \log(m) + B\sqrt{\frac{2\ln\left(2/\delta\right)}{m}}
\]
\end{lemma}
\begin{proof}
Fix a set $A\subset\cx$.
Let $(\cc,\Omega,\mu)$ be a $\left(n\left\lceil \epsilon^{-2} \right\rceil,\epsilon\right)$-compressor for $\ch$.
Let $\tilde\ch$ be the range of $\cc$. Note that $\left|\tilde\ch\right|\le 2^{n\left\lceil \epsilon^{-2} \right\rceil}$. Fix $h\in\ch$. It is enough to show that there is $\tilde h\in \tilde\ch$ with  $\E_{x\in A} \left( h(x) - \tilde h(x)\right)^2\le \epsilon^2$.
Indeed, 
\[
\E_{\omega\sim\mu} \E_{x\in A} \left(h(x) - (\cc_\omega h)(x)\right)^2 =  \E_{x\in A} \E_{\omega\sim\mu}\left(h(x) - (\cc_\omega h)(x)\right)^2 \le \epsilon^2 ~.
\]
Hence, there exists $\tilde h \in \tilde\ch$ for which $\E_{x\in A} \left(h(x) - \tilde h(x)\right)^2\le \epsilon^2$

\end{proof}

\begin{lemma}\label{lem:adl_to_cov_intro}
Fix a class $\ch$ of functions from $\cx$ to $\reals^d$ with approximate description length $n$. Then,
\[
\log\left(N_\infty(\ch,m,\epsilon)\right) \le \log\left(N(\ch,m,\epsilon)\right) \le {n\left\lceil 16\epsilon^{-2} \right\rceil}\lceil\log(dm)\rceil
\]
Hence, if $\ell:\reals^d\times\cy\to \reals$ is $L$-Lipschitz w.r.t. $\|\cdot\|_\infty$ and $B$-bounded, then for any distribution $\cd$ on $\cx\times\cy$
\[
\E_{S\sim\cd^m}\rep_\cd(S,\ch) \lesssim   \frac{(L+B)\sqrt{n\log(dm)} }{\sqrt{m}} \log(m)
\]
Furthermore, with probability at least $1-\delta$,
\[
\rep_\cd(S,\ch) \lesssim   \frac{(L+B)\sqrt{n\log(dm)} }{\sqrt{m}} \log(m) + B\sqrt{\frac{2\ln\left(2/\delta\right)}{m}}
\]
\end{lemma}

\begin{proof}
Denote $k = \lceil\log(dm)\rceil$.
Fix a set $A\subset\cx$.
Let $\cc$ be a $\left(n\left\lceil 16\epsilon^{-2} \right\rceil,\frac{\epsilon}{4}\right)$-compressor for $\ch$.
Define
\[
(\cc'_{\omega_1,\ldots,\omega_k}h)(x) = \med\left((\cc_{\omega_1}h)(x),\ldots,(\cc_{\omega_k}f)(x)  \right)
\]
Let $\tilde\ch$ be the range of $\cc'$. Note that $\left|\tilde\ch\right|\le 2^{kn\left\lceil 16\epsilon^{-2} \right\rceil}$. Fix $h\in\ch$. It is enough to show that there is $\tilde h\in \tilde\ch$ with  $\max_{x\in A} \left\|h(x) - \tilde h(x)\right\|_{\infty}\le \epsilon$.
By lemma \ref{lem:median} we have that
\[
\Pr_{\omega_1,\ldots,\omega_k\sim\mu}\left( \exists x\in A,\;\left|(\cc'_{\omega_1,\ldots,\omega_k}h)(x) - h(x)  \right| > \epsilon  \right) < dm 2^{-k} \le 1
\]
In particular, there exists $\tilde h \in \tilde\ch$ for which $\max_{x\in A} \left\|h(x) - \tilde h(x)\right\|_{\infty}\le \epsilon$

\end{proof}

\subsection{Linear Functions}
We next bound the approximate description length of linear functions with bounded Frobenius norm. 
\begin{theorem}\label{thm:linear_intro}
Let class $\cl_{d_1,d_2,M} =\left\{ \x\in\ball^{d_1}\mapsto W\x : W\text{ is }d_2\times d_1\text{ matrix with }\|W\|_F\le M \right\}$ has approximate description length
\[
n\le  \left\lceil \frac{1}{4} + 2M^2\right\rceil 2 \left\lceil \log\left(  2d_1d_2(M+1) \right)   \right\rceil
\]
Hence, if $\ell:\reals^{d_2}\times\cy\to \reals$ is $L$-Lipschitz w.r.t. $\|\cdot\|_\infty$ and $B$-bounded, then for any distribution $\cd$ on $\cx\times\cy$
\[
\E_{S\sim\cd^m}\rep_\cd(S,\cl_{d_1,d_2,M}) \lesssim   \frac{(L+B)\sqrt{M^2\log(d_1d_2M)\log(d_2m)} }{\sqrt{m}} \log(m)
\]
Furthermore, with probability at least $1-\delta$,
\[
\rep_\cd(S,\cl_{d_1,d_2,M}) \lesssim   \frac{(L+B)\sqrt{M^2\log(d_1d_2M)\log(d_2m)} }{\sqrt{m}} \log(m) + B\sqrt{\frac{2\ln\left(2/\delta\right)}{m}}
\]
\end{theorem}
We remark that the above bounds on the representativeness coincides with standard bounds (\cite{shalev2014understanding} for instance), up to log factors. The advantage of these bound is that they remain valid for {\em any output dimension $d_2$}.

In order to prove theorem \ref{thm:linear_intro} we will use a randomized sketch of a matrix.
\begin{definition}
Let $\bw\in\reals^d$ be a vector. A {\em random sketch of $\bw$} is a random vector $\hat \bw$ that is samples as follows.  Choose $i$ w.p.
$p_{i} = \frac{w_{i}^2}{2\| \bw\|^2} + \frac{1}{2d}$. Then, w.p. $ \frac{w_{i}}{p_{i}} - \left\lfloor \frac{w_{i}}{p_{i}} \right\rfloor$ let $b=1$ and otherwise $b=0$.
Finally, let $\hat \bw = \left(\left\lfloor \frac{w_{i}}{p_{i}} \right\rfloor + b\right) \be_{i}$. A {\em random $k$-sketch of $\bw$} is an average of $k$-independent random sketches of $\bw$. A random sketch and a random $k$-sketch of a matrix is defined similarly, with the standard matrix basis instead of the standard vector basis.
\end{definition}
The following useful lemma shows that an sketch $\bw$ is a $\sqrt{ \frac{1}{4} + 2\|\bw\|^2}$-estimator of $\bw$.
\begin{lemma}\label{lem:sketch}
Let $\hat \bw$ be a random sketch of $\bw\in\reals^d$. Then,
\begin{enumerate}
\item
$\E\hat\bw = \bw$
\item
For any $\bu\in\sphere^{d-1}$, $\E\left(\inner{\bu,\hat \bw} -\inner{\bu, \bw} \right)^2 \le  \E\inner{\bu,\hat \bw}^2 \le \frac{1}{4} + 2\|\bw\|^2$
\end{enumerate}
\end{lemma}
\begin{proof}
Items 1. is straight forward. To see item 2. note that
\begin{eqnarray*}
\E\left(\inner{\bu,\hat \bw} -\inner{\bu, \bw} \right)^2 &\le & \E\inner{\bu,\bw}^2
\\
 &=& \sum_{i}p_{i} \left(   \left[ \frac{w_{i}}{p_{i}} \right]\left(\left\lfloor \frac{w_{i}}{p_{i}} \right\rfloor + 1\right)^2  + 
\left(1 -  \left[ \frac{w_{i}}{p_{i}} \right]\right)
\left(\left\lfloor \frac{w_{i}}{p_{i}} \right\rfloor\right)^2   \right) u_i^2
\\
&=& \sum_{i}p_{i} \left(\left(\left\lfloor \frac{w_{i}}{p_{i}} \right\rfloor \right)^2+ 2\left[ \frac{w_{i}}{p_{i}} \right]\left\lfloor \frac{w_{i}}{p_{i}} \right\rfloor 
+ \left[ \frac{w_{i}}{p_{i}} \right]\right) u_i^2
\\
&=& \sum_{i}p_{i} \left(\left(\left\lfloor \frac{w_{i}}{p_{i}}\right\rfloor + \left[ \frac{w_{i}}{p_{i}} \right] \right)^2+ \left[ \frac{w_{i}}{p_{i}} \right] - \left[ \frac{w_{i}}{p_{i}} \right]^2\right) u_i^2
\\
&=& \sum_{i}p_{i} \left(\left(\frac{w_{i}}{p_{i}} \right)^2+ \left[ \frac{w_{i}}{p_{i}} \right]\left(1 - \left[ \frac{w_{i}}{p_{i}} \right]\right)\right) u_i^2
\\
&\le& \sum_{i}p_{i} \left(\left(\frac{w_{i}}{p_{i}} \right)^2+ \frac{1}{4}\right) u_i^2
\\
&\le& \frac{1}{4} \|\bu\|_{\infty}^2 + \sum_{i} \frac{w^2_iu_i^2}{p_{i}} 
\\
&\le& \frac{1}{4}  + \sum_{i} \frac{w^2_iu_i^2}{p_{i}} 
\end{eqnarray*}
Now, since $p_{i} = \frac{w_{i}^2}{2\| \bw\|^2} + \frac{1}{2d}$ we have
\[
\sum_{i} \frac{w^2_iu_i^2}{p_{i}} \le \sum_{i} \frac{w^2_{i}u_i^2}{\frac{w_{i}^2}{2\| \bw\|^2} } = 2\|\bw\|^2\sum_i u_i^2 =  2\|\bw\|^2
\]

\end{proof}

\begin{proof} (of theorem \ref{thm:linear_intro})
We construct a compressor for $\cl_{d_1,d_2,M}$ as follows. 
Given $W$, we will sample a $k$-sketch $\hat W$ of $W$ for $k = \left\lceil\frac{1}{4} + 2M^2\right\rceil$, and will return the function $\x\mapsto \hat W\x$. 
We claim that that $W\mapsto\hat W$ is a $(1, 2k \left\lceil\log(2d_1d_2(M+1))\right\rceil)$-compressor for $\cl_{d_1,d_2,M}$. Indeed,  to specify a sketch of $W$ we need $\left\lceil\log(d_1d_2)\right\rceil$ bits to describe the chosen index, as  well as $\log\left(2d_1 d_2 M + 2\right)$ bits to describe the value in that index. Hence, $2k \left\lceil\log(2d_1d_2(M+1))\right\rceil$ bits suffices to specify a $k$-sketch. It remains to show that for $\x\in \ball^{d_1}$, $\hat W\x$ is a $1$-estimator of $W\x$. Indeed, by lemma \ref{lem:sketch}, $\E\hat W = W$ and therefore
\[
\E \hat W\x = W\x
\]
Likewise, for $\bu\in\sphere^{d_2-1}$. We have
\[
\E\left(\inner{\bu,\hat W\x} -\inner{\bu, W\x} \right)^2  = \E\left(\inner{\hat W, \x\bu^T} -\inner{W,\x\bu^T} \right)^2 \le \frac{\frac{1}{4} + 2M^2}{k} \le 1
\]
\end{proof}

\subsection{Simplified Depth $2$ Networks}
To demonstrate our techniques, we consider the following class of functions. We let the domain $\cx$ to be $\ball^d$. We fix an activation function $\rho:\reals\to\reals$ that is assumed to be a polynomial
\[
\rho(x) = \sum_{i=0}^k a_ix^i
\]
with $\sum_{n=1}^n|a_n| = 1$.
For any $W\in M_{d,d}$ we define
\[
h_W(\x) = \frac{1}{\sqrt{d}}\sum_{i=1}^d \rho(\inner{\bw_i, \x})
\]
Finally, we let 
\[
\ch = \left\{ h_W : \forall i,\;\|\bw_i\| \le \frac{1}{2} \right\}
\]
In order to build compressors for classes of networks, we will utilize to compositional structure of the classes. Specifically, we have that
\[
\ch = \Lambda \circ \rho \circ \cf
\]
Where 
\[
\cf = \left\{x\mapsto W\x : W\text{ is }d\times d\text{ matrix with }\|\bw_i\| \le 1\text{ for all }i  \right\}
 \]
and 
\[
\Lambda(\x) = \frac{1}{\sqrt{d}}\sum_{i=1}^dx_i
\] 
As $\cf$ is a subset of $\cl_{d,d,\sqrt{d}}$, we know that there exists a $\left(1, O\left(d\log(d)\right) \right)$-compressor for it. We will use this compressor to build a compressor to $\rho\circ \cf$, and then to $\Lambda \circ \rho \circ \cf$. We will start with the latter, linear case, which is simpler
\begin{lemma}\label{lem:compose_linear_intro}
Let $X$ be a $\sigma$-estimator to $\x\in\reals^{d_1}$. Let $A\in M_{d_2,d_1}$ be a matrix of spectral norm $\le r$. Then, $AX$ is a $(r\sigma)$-estimator to $A\x$.
In particular, if $\cc$ is a $(1,n)$-compressor to a class $\ch$ of functions from $\cx$ to $\reals^d$. Then
\[
\cc'_\omega( \Lambda\circ h) = \Lambda \circ \cc_\omega  h
\]
is a $(1,n)$-compressor to $ \Lambda \circ \ch$
\end{lemma}
\begin{proof}
We have $\E AX = A\E X = A\x$. Furthermore, for any $\bu\in\sphere^{d_2-1}$,
\[
\E\inner{\bu,AX-A\x}^2 = \E\inner{A^T\bu,X-\x}^2 \le \|A^T\bu\|^2\sigma^2 \le r^2\sigma^2
\]
\end{proof}
We next consider the composition of $\cf$ with the non-linear $\rho$. 
As opposed to composition with a linear function, we cannot just generate a compression version using $\cf$'s compressor and then compose with $\rho$. Indeed, if $X$ is a $\sigma$-estimator to $\x$, it is not true in general that $\rho(X)$ is an estimator of $\rho(\x)$. For instance, consider the case that $\rho(x) = x^2$, and $X = (X_1,\ldots,X_d)$ is a vector of independent standard Gaussians. $X$ is a $1$-estimator of $0\in\reals^d$. On the other hand, $\rho(X) = (X^2_1,\ldots,X^2_n)$ is not an estimator of $0 = \rho(0)$.
We will therefore take a different approach. Given $f\in\cf$, we will sample $k$ independent estimators $\{C_{\omega_i}f\}_{i=1}^k$ from $\cf$'s compressor, and define the compressed version of $\sigma\circ h$ as
\[
\cc'_{\omega_1,\ldots,\omega_k}f = \sum_{i=0}^d a_i\prod_{j=0}^i C_{\omega_i}f
\]
This construction is analyzed in the following lemma

\begin{lemma}\label{lem:compose_non_linear_intro}
If $\cc$ is a $\left(\frac{1}{2},n\right)$-compressor of a class $\ch$ of functions from $\cx$ to $\left[-\frac{1}{2},\frac{1}{2}\right]^d$. Then $\cc'$ 
is a $(1,n)$-compressor of $ \rho \circ \ch$
\end{lemma}

Combining theorem \ref{thm:linear_intro} and lemmas \ref{lem:compose_linear_intro}, \ref{lem:compose_non_linear_intro} we have:

\begin{theorem}\label{thm:non_linear_intro}
$\ch$ has approximation length $\lesssim  d\log(d)$.
Hence, if $\ell:\reals\times\cy\to \reals$ is $L$-Lipschitz and $B$-bounded, then for any distribution $\cd$ on $\cx\times\cy$
\[
\E_{S\sim\cd^m}\rep_\cd(S,\ch) \lesssim   \frac{(L+B)\sqrt{d\log(d)} }{\sqrt{m}} \log(m)
\]
Furthermore, with probability at least $1-\delta$,
\[
\rep_\cd(S,\ch) \lesssim  \frac{(L+B)\sqrt{d\log(d)} }{\sqrt{m}} \log(m) + B\sqrt{\frac{2\ln\left(2/\delta\right)}{m}}
\]
\end{theorem}

Lemma \ref{lem:compose_non_linear_intro} is implied by the following useful lemma:
\begin{lemma}\label{lem:estimator_aritmetics}
\begin{enumerate}
\item
If $X$ is a $\sigma$-estimator of $\x$ then $aX$ is a $\left(|a|\sigma\right)$-estimator of $a X$
\item
Suppose that for $n=1,2,3,\ldots$ $X_n$ is a $\sigma_n$-estimator of $\x_n\in\reals^d$.
Assume furthermore that  $\sum_{n=1}^\infty\x_n$ and $\sum_{n=1}^\infty\sigma_n$ converge to $\x\in\reals^d$ and $\sigma\in [0,\infty)$.
Then, $ \sum_{n=1}^\infty X_n$ is a $\sigma$-estimator of $\x$
\item
Suppose that $\{X_i\}_{i=1}^k$ are independent $\sigma_i$-estimators of $\x_i\in\reals^d$. Then $\prod_{i=1}^kX_i$ is a $\sigma'$-estimator of $\prod_{i=1}^k\x_i$  for $\sigma'^2 = \prod_{i=1}^k\left(\sigma_i^2 + \left\|\x_i\right\|_\infty^2\right)  - \prod_{i=1}^k\left\|\x_i\right\|_\infty^2$
\end{enumerate}
\end{lemma}
We note that the bounds in the above lemma are all tight. Specifically, (3) is tight in the case that $\{X_i\}_{i=1}^k$ are independent Gaussians with means $\{\x_i\}_{i=1}^k$ and co-variance matrices $\{\sigma^2_iI\}_{i=1}^k$.

\begin{proof} 
1. and 2. are straight forward. We next prove 3. By replacing each $X_i$ with $\frac{X_i}{\sigma_i}$ we can assume w.l.o.g. that $\sigma_1=\ldots=\sigma_k=1$. We have
\scriptsize
\begin{eqnarray*}
\E_{X_1,\ldots, X_k}\inner{\bu,\prod_{i=1}^kX_i - \prod_{i=1}^k\x_i}^2 &=& \E_{X_1,\ldots, X_k}\inner{\bu,\prod_{i=1}^k\left(\left(X_i-\x_i\right) + \x_i\right) - \prod_{i=1}^k\x_i}^2
\\
&=& \E_{X_1,\ldots, X_k}\inner{\bu,\sum_{A\subset [k]}\prod_{i\in A}\left(X_i-\x_i\right) \prod_{i\in A^c}\x_i - \prod_{i=1}^k\x_i}^2
\\
&=& \E_{X_1,\ldots, X_k}\left(\inner{\bu,\sum_{A\subset [k]}\prod_{i\in A}\left(X_i-\x_i\right) \prod_{i\in A^c}\x_i} - \inner{\bu,\prod_{i=1}^k\x_i}\right)^2
\\
&=& \E_{X_1,\ldots, X_k}\sum_{A\subset [k]}\sum_{B\subset [k]}   \inner{\bu,\prod_{i\in A}\left(X_i-\x_i\right) \prod_{i\in A^c}\x_i}\inner{\bu,\prod_{i\in B}\left(X_i-\x_i\right) \prod_{i\in B^c}\x_i}
\\
&& -2\E_{X_1,\ldots, X_k}\sum_{A\subset [k]}    \inner{\bu, \prod_{i=1}^k\x_i} \inner{\bu,\prod_{i\in A}\left(X_i-\x_i\right) \prod_{i\in A^c}\x_i}
\\
&& + \inner{\bu, \prod_{i=1}^k\x_i}^2
\\
&\stackrel{(1)}{=}& \E_{X_1,\ldots, X_k}\sum_{A\subset [k]}   \inner{\bu,\prod_{i\in A}\left(X_i-\x_i\right) \prod_{i\in A^c}\x_i}^2 -\inner{\bu, \prod_{i=1}^k\x_i}^2
\\
&\stackrel{(2)}{\le}& \sum_{A\subset [k], A\ne [k]}   \left\|\bu\prod_{i\in A^c}\x_i\right\|^2 
\\
&=& \sum_{A\subset [k], A\ne \emptyset}   \left\|\bu\prod_{i\in A}\x_i\right\|^2 
\\
&\stackrel{(3)}{\le}& \sum_{A\subset [k], A\ne \emptyset}   \prod_{i\in A}\left\|\x_i\right\|_\infty^2 
\\
&=&  \prod_{i=1}^k\left(1 + \left\|\x_i\right\|_\infty^2\right)  - \prod_{i=1}^k\left\|\x_i\right\|_\infty^2
\end{eqnarray*}
\normalsize
\begin{enumerate}
\item[(1)] If $A\ne B$, then w.l.o.g.  $k\in A\setminus B$. In this case we have
\scriptsize
\begin{eqnarray*}
 && \E_{X_1,\ldots, X_k}\inner{\bu,\prod_{i\in A}\left(X_i-\x_i\right) \prod_{i\in A^c}\x_i}\inner{\bu,\prod_{i\in B}\left(X_i-\x_i\right) \prod_{i\in A^c}\x_i}
 \\ &=&
 \E_{X_1,\ldots, X_{k-1}}\E_{X_k}\inner{\bu,\prod_{i\in A}\left(X_i-\x_i\right) \prod_{i\in A^c}\x_i}\inner{\bu,\prod_{i\in B}\left(X_i-\x_i\right) \prod_{i\in B^c}\x_i}
\\ &=&
 \E_{X_1,\ldots, X_{k-1}}\inner{\bu,\prod_{i\in B}\left(X_i-\x_i\right) \prod_{i\in B^c}\x_i}\E_{X_k}\inner{\bu,\prod_{i\in A}\left(X_i-\x_i\right) \prod_{i\in A^c}\x_i}
\\ &=&
 \E_{X_1,\ldots, X_{k-1}}\inner{\bu,\prod_{i\in B}\left(X_i-\x_i\right) \prod_{i\in B^c}\x_i}\inner{\bu,\prod_{i\in A\setminus [k]}\left(X_i-\x_i\right)\stackrel{=0}{\overbrace{\E_{X_k}(X_k-\x_k)}} \prod_{i\in A^c}\x_i}
\\ &=& 0
\end{eqnarray*}
\normalsize
Similarly, if $A\ne\emptyset$, then w.l.o.g.  $k\in A$. In this case we have
\scriptsize
\begin{eqnarray*}
\E_{X_1,\ldots, X_k}    \inner{\bu, \prod_{i=1}^k\x_i} \inner{\bu,\prod_{i\in A}\left(X_i-\x_i\right) \prod_{i\in A^c}\x_i} &=&\E_{X_1,\ldots, X_{k-1}}\E_{X_k}    \inner{\bu, \prod_{i=1}^k\x_i} \inner{\bu,\prod_{i\in A}\left(X_i-\x_i\right) \prod_{i\in A^c}\x_i}
\\
&=& \E_{X_1,\ldots, X_{k-1}}\inner{\bu, \prod_{i=1}^k\x_i}\inner{\bu,\prod_{i\in A\setminus [k]}\left(X_i-\x_i\right)\stackrel{=0}{\overbrace{\E_{X_k}(X_k-\x_k)}} \prod_{i\in A^c}\x_i}
\end{eqnarray*}
\normalsize
\item[(2)] Fix a set $A$ that is w.l.o.g. $A=\{1,\ldots,k'\}$. We note that if $X\in\reals^d$ is a $1$-estimator to $0$, then for any vector $\z\in\reals^d$
\[
\E_X\|\z X\|^2 = \sum_{i=1}^d z_i^2\E_X X_i^2 = \sum_{i=1}^d z_i^2\E_X\inner{\be_i,X}^2  \le \sum_{i=1}^d z_i^2 = \|\z\|^2
\]
It follows that
\begin{eqnarray*}
\E_{X_1,\ldots, X_{k'-1}}\left\| \bu\prod_{i = k'+1}^k\x_i\prod_{i=1}^{k'-1}\left(X_i-\x_i\right) \right\|^2 &=&\E_{X_1,\ldots, X_{k'-2}}\E_{X_{k'-1}}\left\| \bu\prod_{i = k'+1}^k\x_i\prod_{i=1}^{k'-1}\left(X_i-\x_i\right) \right\|^2
\\ &\le &
\E_{X_1,\ldots, X_{k'-2}}\left\| \bu\prod_{i = k'+1}^k\x_i\prod_{i=1}^{k'-2}\left(X_i-\x_i\right) \right\|^2
\\ & \vdots &
\\ &\le &
\left\| \bu\prod_{i = k'+1}^k\x_i \right\|^2
\\ &= &
\left\| \bu\prod_{i \in A^c}\x_i \right\|^2
\end{eqnarray*}
Hence,
\scriptsize
\begin{eqnarray*}
\E_{X_1,\ldots, X_k}   \inner{\bu,\prod_{i\in A}\left(X_i-\x_i\right) \prod_{i\in A^c}\x_i}^2  &=&  \E_{X_1,\ldots, X_{k'-1}}\E_{X_{k'}}   \inner{\bu\prod_{i = k'+1}^k\x_i\prod_{i=1}^{k'-1}\left(X_i-\x_i\right) , \left(X_{k'}-\x_{k'}\right)}^2
\\
&\stackrel{X_{k'}\text{ is }1\text{-estimator of }\x_k}{\le} & \E_{X_1,\ldots, X_{k'-1}} \left\| \bu\prod_{i = k'+1}^k\x_i\prod_{i=1}^{k'-1}\left(X_i-\x_i\right) \right\|^2
\\
&\le &  \left\| \bu\prod_{i \in A^c}\x_i \right\|^2
\end{eqnarray*}
\normalsize
\item[(3)] If $\z = \bu\prod_{i\in A}\x_i$ then for any $j\in[d]$, $|z_j| \le |u_j| \prod_{i\in A}\|\x_i\|_\infty$. Hence,
\[
\|\z\|^2 \le \prod_{i\in A}\|\x_i\|_\infty \sum_{j=1}^d u_j^2 = \prod_{i\in A}\|\x_i\|_\infty
\]

\end{enumerate}

\end{proof}

\section{Approximation Description Length}
In this section we refine the definition of approximate description length that were given in section \ref{sec:adl_intro}.
We start with the encoding of the compressed version of the functions. Instead of standard strings, we will use what we call {\em bracketed string}.  The reason for that often, in order to create a compressed version of a function, we concatenate compressed versions of other functions. 
This results with strings with a nested structure. For instance, consider the case that a function $h$ is encoded by the concatenation of $h_1$ and $h_2$. Furthermore, assume that $h_1$ is encoded by the string $01$, while $h_2$ is encoded by the concatenation of $h_3, h_4$ and $h_5$ that are in turn encoded by the strings $101$, $0101$ and $1110$. The encoding of $h$ will then be
\[
[[01][ [101] [0101] [1110] ]]
\]
We note that in section \ref{sec:adl_intro} we could avoid this issue since the length of the strings and the recursive structure were fixed, and did not depend on the function we try to compress. Formally, we define

\begin{definition}
A {\em bracketed string} is a rooted tree $S$, such that (i) the children of each edge are ordered, (ii) there are no nodes with a singe child, and (iii) the leaves are labeled by $\{0, 1\}$. The {\em length}, $\len(S)$ of $S$ is the number of its leaves. 
\end{definition}
Let $S$ be a bracketed string. There is a linear order on its leaves that is defined as follows. Fix a pair of leaves, $v_1$ and $v_2$,  and let $u$ be their 
LCA. Let $u_1$ (resp. $u_2$) be the child of $u$ that lie on the path to $v_1$ (resp. $v_2$). We define $v_1<v_2$ if $u_1<u_2$ and $v_1>v_2$ otherwise (note that necessarily $u_1\ne u_2$). Let $v_1,\ldots,v_n$ be the leaves of $T$, ordered according to the above order, and let $b_1,\ldots,b_n$ be the corresponding bits. The string associated with $T$ is $s = b_1\ldots b_n$. We denote by $\cs_n$ the collection of bracketed strings of length $\le n$, and by $\cs = \cup_{n=1}^\infty\cs_n$ the collection of all bracketed strings.

The following lemma shows that in log-scale, the number of bracketed strings of length $\le n$ differ from standard strings of length $\le n$ by only a constant factor 
\begin{lemma}
$|\cs_n|\le 32^n$
\end{lemma}
\begin{proof}
By adding a pair of brackets around each bit, each bracketed string can be described by $2n-1$ correctly matched pairs of brackets, and a string of length $\le n$. As the number of ways to correctly match $k$ pairs of brackets is the Catalan number  $C_k = \frac{1}{k+1}\binom{2k}{k} \le 2^{2k}$, we have, $|\cs_n|\le 2^{4n-2}2^{n+1}$
\end{proof}

We next revisit the definition of a compressor for a class $\ch$. The definition of compressor  will now have a third parameter, $n_s$, in addition to $\sigma$ and $n$.
We will make three changes in the definition. The first, which is only for the sake of convenience, is that we will use bracketed strings rather than standard strings.   
The second change, is that the length of the encoding string will be bounded only {\em in expectation}. The final change is that the compressor  can now output a {\em seed}. That is, given a function $h\in \ch$ that we want to compress, the compressor can generate both a non-random seed $E_s(h)\in \cs_{n_s}$ and a random encoding $E(\omega,h)\in \cs$ with $\E_{\omega\sim\mu}\len(E(\omega, h)) \le n$.  Together, $E_s(h)$ and $E(\omega,h)$ encode a $\sigma$-estimator. Namely, there is a function $D:\cs_{n_s}\times\cs \to \left(\reals^d\right)^\cx$ such that $D(E_s(h),E(\omega, h)),\;\omega\sim\mu$ is a $\sigma$-estimator of $h$. The advantage of using seeds is that it will allow us to generate many independent estimators, at a lower cost. In the case that $n\ll n_s$, the cost of generating $k$ independent estimators of $h\in\ch$ is $n_s + kn$ bits (in expectation) instead of $k(n_s+n)$ bits. Indeed, we can encode $k$ estimators by a single seed $E_s(h)$ and $k$ independent ``regular" encodings $E(\omega_k, h),\ldots, E(\omega_k, h)$. The formal definition is given next.

\begin{definition}
A {\em $(\sigma, n_s, n)$-compressor} for $\ch$ is a $5$-tuple  $\cc = (E_s,E,D,\Omega,\mu)$ where $\mu$ is a probability measure on  $\Omega$, and $E_s,E,D$ are functions $E_s: \ch \to \ct^{n_s}$, $E:\Omega\times \ch \to \ct$,  and $D:\ct^{n_s}\times\ct \to \left(\reals^d\right)^\cx$ such that for any $h\in\ch$ and $x\in\cx$
\begin{enumerate}
\item
$D(E_s(h),E(\omega, h)),\;\omega\sim\mu$ is a $\sigma$-estimator of $h$.
\item
$\E_{\omega\sim\mu}\len(E(\omega, h)) \le n$
\end{enumerate}
\end{definition}

We finally revisit the definition of approximate description length. We will add an additional parameter, to accommodate the use of seeds. Likewise, the approximate description length will now be a function of $m$ -- we will say that $\ch$ has approximate description length $(n_s(m),n(m))$ if there is a $(1, n_s(m), n(m))$-compressor for the restriction of $\ch$ to any set  $A\subset\cx$ of size at most $m$. Formally:

\begin{definition}
We say that a class $\ch$ of functions from $\cx$ to $\reals^d$ has {\em approximate description length} $(n_s(m),n(m))$ if for any set $A\subset\cx$ of size $\le m$ there exists a $(1, n_s(m), n(m))$-compressor for $\ch|_A$ 
\end{definition}
It is not hard to see that if $\ch$ has approximate description length $(n_s(m),n(m))$, then for any $1\ge \epsilon>0$ and a set $A\subset\cx$ of size $\le m$, there exists an $\left(\epsilon, n_s(m), n(m)\lceil \epsilon^{-2} \rceil\right)$-compressor for $\ch|_A$. 
We next connect the approximate description length, to covering numbers and representativeness. The proofs are similar the the proofs of lemmas \ref{lem:adl_to_cov_1d_intro} and \ref{lem:adl_to_cov_intro}.

\begin{lemma}\label{lem:adl_to_cov_1d}
Fix a class $\ch$ of functions from $\cx$ to $\reals$  with approximate description length $(n_s(m),n(m))$. Then,
\[
\log\left(N(\ch,m,\epsilon)\right) \lesssim n_s(m) + \frac{n(m)}{\epsilon^2}
\]
Hence, if $\ell:\reals^d\times\cy\to \reals$ is $L$-Lipschitz and $B$-bounded, then for any distribution $\cd$ on $\cx\times\cy$
\[
\E_{S\sim\cd^m}\rep_\cd(S,\ch) \lesssim   \frac{(L+B)\sqrt{n_s(m) + n(m)} }{\sqrt{m}} \log(m)
\]
Furthermore, with probability at least $1-\delta$,
\[
\rep_\cd(S,\ch) \lesssim   \frac{(L+B)\sqrt{n_s(m)+n(m)} }{\sqrt{m}} \log(m) + B\sqrt{\frac{2\ln\left(2/\delta\right)}{m}}
\]
\end{lemma}

\begin{lemma}\label{lem:adl_to_cov}
Fix a class $\ch$ of functions from $\cx$ to $\reals^d$ with approximate description length $(n_s(m),n(m))$. Then,
\[
\log\left(N(\ch,m,\epsilon)\right)\le \log\left(N_\infty(\ch,m,\epsilon)\right) \lesssim n_s(m) + \frac{ n(m) \log(dm)}{\epsilon^2}
\]
Hence, if $\ell:\reals^d\times\cy\to \reals$ is $L$-Lipschitz w.r.t. $\|\cdot\|_\infty$ and $B$-bounded, then for any distribution $\cd$ on $\cx\times\cy$
\[
\E_{S\sim\cd^m}\rep_\cd(S,\ch) \lesssim   \frac{(L+B)\sqrt{n_s(m) + n(m)\log(dm)} }{\sqrt{m}} \log(m)
\]
Furthermore, with probability at least $1-\delta$,
\[
\rep_\cd(S,\ch) \lesssim   \frac{(L+B)\sqrt{n_s(m) + n(m)\log(dm)} }{\sqrt{m}} \log(m) + B\sqrt{\frac{2\ln\left(2/\delta\right)}{m}}
\]
\end{lemma}

\subsection{Linear Operations}

\begin{lemma}\label{lem:adl_lin_1}
Let $\ch_1,\ch_2$ be classes of functions from $\cx$ to $\reals^d$ with approximate description length of $(n^1_s(m),n^1(m))$ and $(n^2_s(m),n^2(m))$. Then $\ch_1 + \ch_2$ has approximate description length of $(n^1_s(m)+ n^2_s(m),2n^1(m) + 2n^2(m))$
\end{lemma}

\begin{lemma}\label{lem:adl_lin_2}
Let $\ch$ be a class of functions from $\cx$ to $\reals^d$ with approximate description length of $(n_s(m),n(m))$. Let $A$ be $d_2\times d_1$ matrix.
Then $A\circ \ch_1$ has approximate description length $ \left(n_s(m), \left\lceil \|A\|^2\right\rceil  n(m)\right)$
\end{lemma}

\begin{definition}
Denote by $\cl_{d_1,d_2,r,R}$ the class of all $d_2\times d_1$ matrices of spectral norm at most $r$ and Frobenius norm at most $R$.
\end{definition}

\begin{lemma}\label{lem:adl_lin_3}
Let $\ch$ be a class of functions from $\cx$ to $\reals^{d_1}$ with approximate description length $(n_s(m), n(m))$. Assume furthermore that for any $x\in \cx$ and $h\in\ch$  we have that $\|h(x)\|\le B$.
Then, $\cl_{d_1,d_2,r,R}\circ\ch$ has approximate description length
\[
\left(n_s(m), n(m)O(r^2+1) + O\left((d_1 + B^2)(R^2 + 1)\log(Rd_1d_2 + 1)\right)\right)
\]
\end{lemma}

\begin{proof}
Fix as set $A\subset\cx $ of size $m$.
We will construct a compressor to $\cl_{d_1,d_2,r,R}\circ\ch$ as follows. Given $h\in\ch$ and $W\in  \cl_{d_1,d_2,r,R}$ we first pay a seed cost  $n_s(m)$ to use $\ch$'s compressor. 
Then, we use $\ch$'s compressor to generate a $\sqrt{\frac{1}{k_1}}$-estimator $\hat h$ of $h$, at the cost of $k_1n(m)$ bits. Then, we take $\hat W$ to be a $k_2$-sketch of $W$, at the costs of $k_2O\left(\log\left(d_1 d_2 R + 1\right)\right)$ bits.
Finally, we output the estimator $\hat h \circ \hat W$. Fix $a\in A$. We must show that $\hat WX:= \hat W\hat h(a)$ is a $1$-estimator of $\x = h(a)$. Indeed, for $\bu\in\sphere^{d_2-1}$ we have,
\scriptsize
\begin{eqnarray*}
\E_X\E_{\hat W} \inner{\bu, \hat W X- W\x}^2 &=&\E_X\E_{\hat W} \inner{\bu, \hat W X- WX}^2 + 2 \inner{\bu, \hat W X- WX}\inner{\bu,  W X- W\x} +  \inner{\bu,  W X- W\x}^2
\\
&=&\E_X\E_{\hat W} \inner{\bu, \hat W X- WX}^2 + 2 \E_X\inner{\bu, \stackrel{=0}{\overbrace{\E_{\hat W}\left[\hat W - W\right]}}X}\inner{\bu,  W X- W\x} +  \E_X\E_{\hat W}\inner{\bu,  W X- W\x}^2
\\
&=&\E_X\E_{\hat W} \inner{\bu, \hat W X- WX}^2 +  \inner{\bu,  W X- W\x}^2
\\
&=&\E_X\E_{\hat W} \inner{\hat W - W,   X\bu^T}^2 +  \inner{W^T\bu,   X- \x}^2
\\
&\stackrel{\text{Lemma }\ref{lem:sketch}}{\le} &\frac{2\|W\|^2_F + 1}{k_2}\E_X \|X\|^2  + \frac{1}{k_1}\|W\bu\|^2
\\
& \stackrel{(1)}{\le} & \frac{2\|W\|^2_F + 1}{k_2}\left[ \E_X \|X - \x\|^2 + \|\x\|^2 \right]+ \frac{1}{k_1}\|W\|^2
\\
&\stackrel{(2)}{\le}  &  \frac{2\|W\|^2_F + 1}{k_2}\left[ \frac{1}{k_1} d_1 + \|\x\|^2 \right]+ \frac{1}{k_1}\|W\|^2
\\
&\le &  \frac{2R^2 + 1}{k_2}\left[ \frac{1}{k_1} d_1 + B^2 \right]+ \frac{1}{k_1}r^2
\end{eqnarray*}
\normalsize
\begin{enumerate}
\item[(1)]  We have
\begin{eqnarray*}
\E_X \|X - \x\|^2 &=& \E_X \|X \|^2 - 2\inner{X,\x} +  \| \x\|^2 
\\& =&
 \E_X \|X \|^2 - 2\inner{\E X,\x} +  \| \x\|^2 
 \\&=&
  \E_X \|X \|^2 -   \| \x\|^2
\end{eqnarray*}
\item[(2)]  We have
\begin{eqnarray*}
\E_X \|X - \x\|^2 &=& \sum_{i=1}^{d_1} \E (X_i - x_i)^2
\\& =&
 \sum_{i=1}^{d_1} \E \inner{X - \x,\be_i}^2
 \\&\le &
   \sum_{i=1}^{d_1} \frac{1}{k_1} = \frac{d_1}{k_1}
\end{eqnarray*}
\end{enumerate}
Finally, by choosing $k_1 = \left\lceil2r^2\right\rceil + 1$ and $k_2 = 2(d_1 + B^2)(2R^2 + 1)$ we get the result.
\end{proof}

\subsection{Non-Linear Operations}
\begin{lemma}\label{lem:extimator_of_omposition}
Suppose that $\{X_n\}_{n=1}^\infty$ are independent $\sigma$-estimators to $\x\in\reals^d$.
Let $\rho(\bt) =  \sum_{n=0}^\infty \ba_n \bt^{ n}$. Let $U = \ba_0 +  \sum_{n=1}^\infty \hat \ba_n Y_n$ where $Y_n = \prod_{i=1}^nX_i$ and $\hat \ba_n = \frac{\ba_n}{p_1}$ w.p. $p_i$ and $0$ otherwise. Then $U$ is $\sigma$'-estimator of $\rho(\x)$ with $\sigma' = \sum_{n=1}^\infty \sqrt{\frac{\|\ba_n\|_\infty^{2}}{p_n}\left(\left(\sigma^2 +\|\x\|_\infty^2\right) ^{n} + (1-p_n)d\|\x\|_\infty^{2n}\right)} $.
\end{lemma}
\begin{remark}
In particular, if  $\|\ba_n\|_\infty\le B^n$, $\sqrt{\sigma^2 + \|\x\|^2_\infty}\le \frac{1}{6B}$ and $p_n = \begin{cases}1 & n\le \left\lceil\frac{\log_3(d)}{2}\right\rceil \\ 4^{-n}&\text{otherwise}\end{cases}$, We have $\sigma'\le 1$ and $\E\max\{n : \hat \ba_n\ne 0\}\le  \frac{\log_3(d) + 4}{2}$. Indeed,
\scriptsize
\begin{eqnarray*}
 \sum_{n=1}^\infty \sqrt{\frac{\|\ba_n\|_\infty^{2}}{p_n}\left(\left(\sigma^2 +\|\x\|_\infty^2\right) ^{n} + (1-p_n)d\|\x\|_\infty^{2n}\right)}
 &\le &\sum_{n=1}^\infty \sqrt{\frac{\|\ba_n\|_\infty^{2}}{p_n}\left(\sigma^2 +\|\x\|_\infty^2\right) ^{n}} + \sum_{n=1}^\infty\sqrt{\frac{\|\ba_n\|_\infty^{2}}{p_n}(1-p_n)d\|\x\|_\infty^{2n}}
 \\
&\le &\sum_{n=1}^\infty \left(2B\right)^n\sqrt{\left(\sigma^2 +\|\x\|_\infty^2\right) ^{n}} + \sqrt{d}\sum_{n=\left\lceil \frac{\log_3(d)}{2}\right\rceil +1}^\infty \left(2B\right)^n \|\x\|^n_\infty
\\
&\le & 
\sum_{n=1}^\infty \left(\frac{1}{3}\right)^{n} + \sqrt{d}\sum_{n=\left\lceil \frac{\log_3(d)}{2}\right\rceil+1}^\infty \left(\frac{1}{3}\right)^{n}
\\
&\le & 
\sum_{n=1}^\infty \left(\frac{1}{3}\right)^{n} + \sum_{n=1}^\infty \left(\frac{1}{3}\right)^{n} = 1
\end{eqnarray*}
\normalsize
and
\begin{eqnarray*}
\E\max\{n : \hat \ba_n\ne 0\}\le\left\lceil \frac{\log_3(d)}{2}\right\rceil + \sum_{n = \left\lceil \frac{\log_3(d)}{2}\right\rceil+1}^\infty 4^{-n}n\le \left\lceil \frac{\log_3(d)}{2}\right\rceil  + 1
\end{eqnarray*}
\end{remark}

\begin{proof}
By lemma \ref{lem:estimator_aritmetics} it is enough to show that for all $n$, $\hat \ba_nY_n$ is a $\sqrt{\frac{\|\ba_n\|_\infty^{2}}{p_n}\left(\left(\sigma^2 +\|\x\|_\infty^2\right) ^{n} + (1-p_n)d\|\x\|_\infty^{2n}\right)} $-estimator of $\ba_n\x^n$. Indeed,
\scriptsize
\begin{eqnarray*}
 \var \left( \inner{\bu, \hat \ba_n Y_n}\right) &=& \E\left(\inner{\bu, \hat \ba_n Y_n} - \inner{\bu,  \ba_n \x^n}\right)^2
 \\
&=& p_n\E\left(\inner{\bu, \frac{ \ba_n}{p_n} Y_n} - \inner{\bu,  \ba_n \x^n}\right)^2 + (1-p_n)\inner{\bu,  \ba_n \x^n}^2
\\
&=& \frac{1}{p_n}\E\inner{\bu,  \ba_n Y_n}^2 - 2\E\inner{\bu,  \ba_n Y_n}\inner{\bu,  \ba_n \x^n}  + p_n\inner{\bu,  \ba_n \x^n}^2 +(1-p_n)\inner{\bu,  \ba_n \x^n}^2
\\
&=& \frac{1}{p_n}\E\inner{\bu,  \ba_n Y_n}^2 -\inner{\bu,  \ba_n \x^n}^2
\\
&=& \frac{1}{p_n}\E\left(\inner{ \ba_n\bu,  Y_n}^2 -\inner{ \ba_n\bu,   \x^n}^2\right) + \frac{1-p_n }{p_n}\inner{\bu,  \ba_n \x^n}^2
\\
&\stackrel{\text{lemma }\ref{lem:estimator_aritmetics}}{\le} &\frac{\|\ba_n\bu\|_2^{2}}{p_n}\left(\left(\sigma^2 +\|\x\|_\infty^2\right) ^{n} + (1-p_n)\|\x^n\|_2^{2}\right) 
\\
&\le &\frac{\|\ba_n\|_\infty^{2}}{p_n}\left(\left(\sigma^2 +\|\x\|_\infty^2\right) ^{n} + (1-p_n)d\|\x\|_\infty^{2n}\right) 
\end{eqnarray*}
\normalsize
\end{proof}

\begin{definition}
A function $f:\reals\to\reals$ is {\em $B$-strongly-bounded} if for all $n\ge 1$, $\|f^{(n)}\|_\infty \le n!B^n $.  Likewise, $f$ is {\em strongly-bounded} if it is $B$-strongly-bounded for some $B$
\end{definition}
We note that
\begin{lemma}
If $f$ is $B$-strongly-bounded then $f$ is analytic and its Taylor coefficients around any point are bounded by $B^n$ 
\end{lemma}

The following lemma gives an example to a strongly bounded sigmoid function, as well as a strongly bounded smoothened version  of the ReLU (see figure \ref{fig:func}).

\begin{figure}\caption{The functions $ \ln\left(1 + e^x\right)$ and $\frac{e^x}{1 + e^x}$}\label{fig:func}
\begin{center}
\includegraphics[scale=0.4]{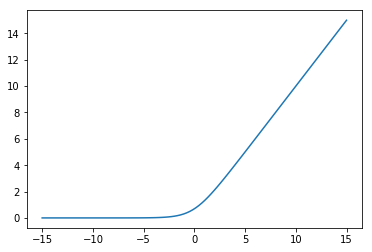}
\includegraphics[scale=0.4]{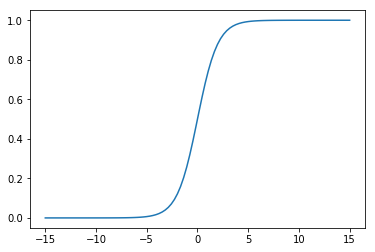}
\end{center}
\end{figure}

\begin{lemma}
The functions $ \ln\left(1 + e^x\right)$ and $\frac{e^x}{1 + e^x}$ are strongly-bounded
\end{lemma}

\begin{proof}
Consider the complex function $f(z) = \frac{e^z}{1 + e^z}$. It is defined in the strip $\{z =x+iy : |y| < \pi\}$. By Cauchy integral formula, for any $r<\pi$, $a\in\reals$ and $n\ge 0$,
\[
f^{(n)}(a) = \frac{n!}{2\pi i} \int_{|z-a| = r}\frac{f(z)}{(z-a)^{n+1}}
\]
It follows that
\[
\left|f^{(n)}(a)\right| \le \frac{n!}{r^n}  \max_{|z-a| = r} |f(z)| \le \frac{n!}{r^n}  \max_{x + iy : |y| < r} |f(x+iy)|
\]
Now, if $|y| < r < \frac{\pi}{2}$, we have
\[
|f(x+iy)| = \frac{e^x}{|1 + e^{iy}e^x|} \le \frac{e^x}{|1 + \cos(y)e^x|} \le \frac{e^x}{|1 + \cos(r)e^x|} \le \frac{1}{\cos(r)}
\]
This implies that  $\frac{e^x}{1 + e^x}$ is strongly bounded. Likewise, since $\frac{e^x}{1 + e^x}$ is the derivative of $ \ln\left(1 + e^x\right)$, the function $ \ln\left(1 + e^x\right)$ is strongly bounded as well.
\end{proof}

\begin{lemma}\label{lem:adl_non_lin}
Let $\ch$ be a class of functions from $\cx$ to $\reals^{d}$ with approximate description length of $(n_s(m),n(m))$. Let $\rho:\reals\to\reals$ be $B$-strongly-bounded.
Then, $\rho\circ\ch$ has approximate description length of 
\[
\left(n_s(m) + O\left(n(m)B^2\log(md)\right),  O\left(n(m)B^2\log(d)\right) \right)
\]
\end{lemma}
\begin{proof}
Fix a set $A\subset\cx$ of size $\le m$.  Let $\epsilon^2 = \sigma^2 = \frac{1}{72B^2}$ and note that $\sqrt{\sigma^2+\epsilon^2}\le\frac{1}{6B}$.
To generate a $1$-estimator to $\rho\circ h\in \rho\circ \ch$ on $A$ we first describe $\tilde h$, which forms the seed, such that $\forall i\in [m],\;\|\tilde h(x_i) - h(x_i) \|_\infty\le \epsilon$. Then, we generate $\sigma$-estimators $\hat h_1, \hat h_2, \ldots,$ to $h|_A$. Finally, we sample Bernoulli random variables $Z_1, Z_2,\ldots$ where the parameter of $Z_n$ is $p_n = \begin{cases}1 & n\le \left\lceil\frac{\log_3(d)}{2}\right\rceil \\ 4^{-n}&\text{otherwise}\end{cases}$. The final estimator is 
\[
\hat g(x) = \rho(\tilde h(x)) +  \sum_{n=1}^\infty  \frac{\rho^{(n)}(\tilde h(x))}{n!}   \frac{Z_n  }{p_n} Y_n\text{ where }Y_n = \prod_{i=1}^n\left(\hat h_i(x) - \tilde h(x)\right)
\]
By lemma \ref{lem:extimator_of_omposition} and the following remark, $\hat g$ is $1$-estimator of $\rho\circ h|_A$.

How many bits do we need in order to specify $\hat g$? 
By lemma \ref{lem:adl_to_cov} the restriction of $\ch|_A$ has  an $\epsilon$-cover, w.r.t. the $\infty$-norm, of log-size $\lesssim n_s(m) + \frac{n(m)\log(md)}{\epsilon^2}$. So the generation of the seed $\tilde h$ costs $n_s(m) + \frac{n(m)\log(md)}{\epsilon^2}$ bits. We also need to specify $N: = \max\{ n : Z_n\ne 0\}$, $Z_1,\ldots,Z_N$ and $\hat h_1,\ldots, \hat h_N$. This can be done by concatenating the descriptions of the pairs $(Z_n,\hat h_n)$ for $n=1,\ldots,N$. The bit cost of this is bounded (in expectation) by $\frac{\log_3(d)+4}{2}\left(\lceil 72B^2 \rceil n(m) + 1\right)$
\end{proof}

\section{Sample Complexity of Neural Networks}

We next utilize the tools we developed in order to analyze the sample complexity of networks. For simplicity we will focus on a standard fully connected architecture. We note that nevertheless the ADL approach is quite flexible, and can be applied to various other network architectures. This is however left for future investigation.
Fix the instance space $\cx$ to be the ball of radius $\sqrt{d}$ in $\reals^d$ (in particular  $[-1,1]^d\subset{\cx}$).
Consider the class
\[
\cn^\rho_{r,R}(d_0,\ldots,d_t) = \left\{W_t\circ\rho\circ W_{t-1}\circ\rho\ldots\circ \rho\circ W_{1} : W_i\in M_{d_{i-1}d_i} \| W_i  \|\le r , \| W_i  \|_F \le R \right\}
\]
and more generally, for matrices $W_i^0 \in M_{d_{i},d_{i-1}},\;i=1,\ldots,t$ consider 
\[
\cn^\rho_{r,R}(W^0_1,\ldots,W^0_t) = \left\{W_t\circ\rho\circ W_{t-1}\circ\rho\ldots\circ \rho\circ W_{1} : \| W_i - W_i^0 \|\le r , \| W_i - W_i^0 \|_F \le R \right\}
\]

\begin{theorem}\label{thm:nn_upper}
Fix a constants\footnote{The constant in the big-O notation will depend only on $t,r$ and $
sigma$.} $r>0$, $t>0$ and a strongly bounded activation $\sigma$.
Then, for every choice of matrices $W_i^0 \in M_{d_{i},d_{i-1}},\;i=1,\ldots,t$ with $d: = \max_{i}d_i$ and $\max_{i}\|W^0_i\| \le  r$ we have that the approximate description length of $\ch = \cn^\rho_{r,R}(W^0_1,\ldots,W^0_t)$ is
\[
\left(dR^2 O\left(\log^{t}(d)\right)\log(md), dR^2 O\left(\log^{t+1}(d)\right) \right)= \left( \tilde{O}\left(dR^2\right), \tilde{O}\left(dR^2\right)\right)
\]
In particular, if $\ell:\reals^{d_t}\times\cy\to \reals$ is bounded and Lipschitz w.r.t. $\|\cdot\|_\infty$, then for any distribution $\cd$ on $\cx\times\cy$
\[
\E_{S\sim\cd^m}\rep_\cd(S,\ch) \le  \tilde O\left( \sqrt{\frac{dR^2  }{m}}  \right)
\]
Furthermore, with probability at least $1-\delta$,
\[
\rep_\cd(S,\ch) \lesssim    \tilde O\left( \sqrt{\frac{dR^2  }{m}}  \right) + O\left(\sqrt{\frac{\ln\left(1/\delta\right)}{m}}\right)
\]
\end{theorem}

The above theorem shows that the sample complexity $\ch$ is $\tilde{O}\left( \frac{dR^2}{\epsilon^2} \right)$.
We next show a corresponding  lower bound.
This lower bound is valid already for the simple case of $\cn^\rho_{1,R}(d,d,1)$, where $\rho$ is the ReLU activation, and  will match this aforementioned bound on the sample complexity up to poly-log factor. However, it will be valid for a family of  activations, that is not the family of strongly-bounded activations, and therefore there is still certain discrepancy between our upper and lower bounds.
The lower bound will be given in the form of {\em shattering}.

\begin{definition}
Let $\ch$ be a class of functions from a domain $X$ to $\reals$. We say that $\ch$ {\em $\gamma$-shatters} a set $A\subset X$ if for any $B \subset A$ there is $h\in\ch$ such that $h|_B\ge \gamma$ while $h|_{A\setminus B}\le -\gamma$. The {\em $\gamma$-fat shattering dimension} of $\ch$, denoted $\fat_\gamma(\ch)$, is the maximal cardinality of a strongly shaterred set.
We will also denote $\fat:=\fat_1$
\end{definition}
It is well known that many losses of interest, such as the large margin loss, ramp loss, the log-loss, the hinge loss, the 0-1 loss  and others, the sample complexity  of a class $\ch$ is lower bounded by $\Omega\left( \frac{\fat(\ch)}{\epsilon^2} \right)$. The following theorem shows that for $R\le \sqrt{d}$, and the ReLU activation $\rho(x)=max(0,x)$,
$\fat\left(\cn^\rho_{1,R}(d,d,1)\right) = \tilde\Omega\left(dR^2\right)$, implying that its sample complexity is $\tilde\Omega\left(\frac{dR^2}{\epsilon^2}\right)$.

\begin{theorem}\label{thm:nn_lower}
Let $\rho$ the ReLU activation. Then, for any $R\le \sqrt{d}$ we have that $\fat(\cn_{1,R}^{\rho}(d,R^2,1)) = \Omega\left(\frac{dR^2}{\log^2(d)}\right)$
\end{theorem}

\subsection{Proof of Theorem \ref{thm:nn_upper}}
We note that 
\[
\cn^\rho_{r,R}(W^0_1,\ldots,W^0_t) = \cn^\rho_{r,R}(W^0_t)\circ\ldots\circ\cn^\rho_{r,R}(W^0_1)
\]
The following lemma analyzes the cost, in terms of approximate description length, when moving from a class $\ch$ to $\cn^\rho_{r,R}(W^0)\circ\ch$.

\begin{lemma}
Let $\ch$ be a class of functions from $\cx$ to $\reals^{d_1}$ with approximate description length $(n_s(m), n(m))$ and  $\|h(x)\|\le M$ for any $x\in \cx$ and $h\in\ch$. Fix $W^0\in M_{d_2,d_1}$. 
Then, $\cn^\rho_{r,R}(W^0_t)\circ\ch$ has approximate description length of 
\[
\left(n_s(m) + n'(m)B^2\log(md_2), n'(m)B^2\log(d_2)\right)
\]
for
\[
n'(m) = n(m)O(r^2+\|W^0\|^2 + 1) + O\left((d_1 + M^2)(R^2 + 1)\log(Rd_1d_2 + 1) \right)
\]
\end{lemma}
The lemma is follows by combining lemmas \ref{lem:adl_lin_1}, \ref{lem:adl_lin_2}, \ref{lem:adl_lin_3} and \ref{lem:adl_non_lin}.
We note that in the case that $d_1, d_2\le d$,  $M = O(\sqrt{d_1})$, $B,r,\|W^0 \|  =O(1)$ (and hence $R = O\left(\sqrt{d}\right)$) and $R\ge 1$ we get that  $\cn^\rho_{r,R}(W^0)\circ\ch$ has approximate description length of 
\[
\left(n_s(m) + O\left(n(m)\log(md)\right), O\left(n(m)\log(d)\right) + O\left(d_1R^2\log^2(d) \right)\right)
\]
Theorem \ref{thm:nn_upper} now follows by simple induction.

\section{Proof of theorem \ref{thm:nn_lower}}

\subsection{Shattering with Quadratic Activation}
In this section we will consider the fat shatering dimension of depth two networks with quadratic activations. We will later use it as a building block for establishing lower bounds on the fat shatering dimension of networks with other activations. Specifically, for $k\le d$ and $B>0$ let $\cq_{d,k,B}$ be the class of functions from the $d$-cube $\{\pm 1\}^d$ to the reals given by 
\[
q(\x) = \sum_{i=1}^k \lambda_i\inner{\bu_i,\x}^2
\]
Where $\bu_1,\ldots,\bu_k$ are orthonormal, and $\max_i |\lambda_i|\le \frac{B}{\sqrt{k}}$. 
We will show that there is a universal constant 
$B>0$ for which $\fat(\cq_{d,k,B}) = \Omega\left(\frac{dk}{\log(d)}\right)$. In fact, we will show 
a slightly stronger result, that will be useful later in order to handle other activation functions. We will use the following notion 
of shattering
\begin{definition}
We say that $\cq_{d,k,B}$ {\em nicely-shatters} the set $A \subset \{\pm 1\}^d$ if $A$ is $1$-shattered by the sub-class
\[
\cq_{d,k,B}(A)=
\left\{q(\x) = \sum_{i=1}^k \lambda_i\inner{\bu_i,\x}^2 \in \cq_{d,k,B} : \forall \x\in A, i\in [k],\; |\inner{\bu_i,\x}|\le \sqrt{2\ln(20d|A|)}  \right\}
\]
\end{definition}

\begin{theorem}\label{thm:quadratic_shatter}
For a universal constant $B>0$, $\cq_{d,k,B}$ with $k\le d$ nicely-shatters a set of size $\Omega\left(\frac{dk}{\log(d)}\right)$
\end{theorem}

Denote by $H_{d,k}$ the space of $d\times d$ symmetric matrices $W$ such that $W_{i,j} = 0$ whenever $\min(i,j)\ge k+1$.
Denote by $\Psi_k:\{\pm 1\}^d\to H_{d,k}$ the mapping
\[
\left(\Psi_k(\x)\right)_{ij} = \begin{cases} \frac{x_ix_j}{\sqrt{k(2d-k)}} & \min(i,j)\le k \\ 0 & \text{otherwise} \end{cases}
\]
We say that a subset $X$ of an inner product space $H$ is $\gamma$-shattered by another subset $F\subset H$, if $X$ is $\gamma$-shattered by the function class $\{\x\mapsto \inner{\f,\x} : \f\in F \}$.
 
\begin{lemma}
Fix $\x_1,\ldots,\x_D\in \{\pm 1\}^d$. Suppose that $\Psi_{k}(\x_1),\ldots,\Psi_k(\x_D)$ are $1$-shaterred by 
\[
\tilde\cq_{d,k,B}(X) =  \{W\in H_{d,k} : \|W\|\le \sqrt{d}B\text{ and } \left|\inner{\bu,\x_i}\right|\le \|\bu\|\sqrt{2\ln(20dD)}\text{ for any $\x_i$ and   eigenvector }\bu\text{ of }W\}
\]
Then, $\x_1,\ldots,\x_D\in \{\pm 1\}^d$ are nicely shaterred by $\cq_{d,\min(2k,d),\sqrt{2}B}$
\end{lemma}
\begin{proof}
Let $h:[D]\to\{\pm 1\}$. There is $W\in H_{d,k}$ such that (1) $\|W\|\le \sqrt{d}B$, (2) $\left|\inner{\bu,\x_i}\right|\le \|\bu\|\sqrt{2\ln(20dD)}$ for any $i\in[D]$ and eigenvector $\bu$ of $W$, and (3) $h(i)\inner{W,\Psi_k(\x_i)}\ge 1$ for any $i\in [D]$.
Let $\bu_1,\ldots,\bu_{k'}$ be normalized and orthonormal sequence of eigenvectors of $W$, that span the space of spanned by the eigenvectors of $W$ corresponding to non-zero eigenvalues. Such a sequence exists since $W$ is symmetric (and is unique, up to sign and order in case that $W$ don't have eigenvalues of multiplicity $>1$).
Since $W$ is of rank at most $\min(2k,d)$, $k'\le \min(2k,d)$. Since $\|W\| \le B\sqrt{d}$, there are scalars $\tilde \lambda_1,\ldots,\tilde \lambda_{k'} \in [-B\sqrt{d},B\sqrt{d}]$ such that $W = \sum_{i=1}^{k'}\tilde\lambda_i \bu_i\bu_i^T$.
Let $\lambda_i = \frac{\tilde\lambda_i}{\sqrt{k(2d-k)}}\in \left[-\frac{B}{\sqrt{k}},\frac{B}{\sqrt{k}}\right]\subset \left[-\frac{\sqrt{2}B}{\sqrt{k'}},\frac{\sqrt{2}B}{\sqrt{k'}}\right]$. We will conclude the proof by showing that $q(\x) = \inner{W,\Psi_{k}(\x)}$ for the function $q\in \cq_{d,\min(2k,d),\sqrt{2}B}(X)$ given by $q(\x) = \sum_{i=1}^{k'} \lambda_i\inner{\bu_i,\x}^2$. Indeed,
\begin{eqnarray*}
\inner{W,\Psi_{k}(\x)} &=& \frac{1}{\sqrt{k(2d-k)}}\sum_{\min(i,j)\le k}W_{ij}x_ix_j
\\
&\stackrel{W\in H_{d,k}}{=}& \frac{1}{\sqrt{k(2d-k)}}\sum_{i,j=1}^dW_{ij}x_ix_j
\\
&=& \frac{1}{\sqrt{k(2d-k)}}\inner{W,\x\x^T}
\\
&=& \frac{1}{\sqrt{k(2d-k)}}\sum_{i=1}^{k'} \tilde{\lambda}_i\inner{\bu_i\bu_i^T,\x\x^T}
\\
&=& \sum_{i=1}^{k'} \lambda_i\inner{\bu_i,\x}^2
\end{eqnarray*}

\end{proof}

Theorem \ref{thm:quadratic_shatter} is therefore implied by the following theorem.

\begin{theorem}\label{thm:quadratic_matrix_shatter}
For a universal constant $B>0$, and any $k\le d$ there is a choice of $D = \Omega\left(\frac{dk}{\log(d)}\right)$ points $\x_1,\ldots,\x_D$ for which $\Psi_{k}(\x_1),\ldots,\Psi_k(\x_D)$ are $1$-shattered by $\tilde\cq_{d,k,B}(X)$
\end{theorem}

The remaining part of this section is devoted to the proof of theorem \ref{thm:quadratic_matrix_shatter}.
We will first show a lemma that shows that any ``large" subset $\cw$ of an inner product space $V$ shatters a contant fraction of any collection of vectors that are ``almost orthogonal". Theorem \ref{thm:quadratic_matrix_shatter} will then follow by showing that there are $\Omega\left(\frac{dk}{\log(d)}\right)$ vectors $\x_1,\ldots\x_D \in \left\{\pm 1\right\}^{d}$ such that $\Psi_k(\x_1),\ldots,\Psi_k(\x_D)$ are ``almost orthonormal" in $H_{d,k}$, and that $\tilde\cq_{d,k,B}(X)$ is a ``large" subset of $H_{d,k}$

Let $V$ be an inner product. We say that a sequence of unit vectors $\x_1,\ldots,\x_D\in V$ is $b$-almost-orthonormal if for any $i\in [D]$, $\|P_{V_{i-1}}\x_i\|^2 \le b$, where $V_i = \mathrm{span}\{\x_j\}_{j=1}^i$ (and $P_{V_i}$ is the orthogonal projection on $V_i$).
In this section we will prove the following lemma:

\begin{lemma}\label{lem:shatter_by_large_measure}
There are universal constants $a,b>0$ for which the following holds.  Let $\x_1,\ldots,\x_D\in V$ be  $\left(\frac{b^2}{2\log(20D)}\right)$-almost-orthonormal and let $\cw\subset V$ be a set of measure $\ge \frac{8}{10}$ according to the standard Gaussian measure on $V$. Then $\cw$ $a$-shatters a set of size $\Omega\left(D\right)$
\end{lemma}

Theorem \ref{thm:quadratic_matrix_shatter} therfore follows from the following two lemmas.

\begin{lemma}\label{lem:random_outer}
Let $V\subset H_{d,K}$ be a linear subspace of dimension $D$. Let $X\in\{\pm 1\}^{d}$ be a uniform vector. Then
\[
\E\|P_V\Psi_k(X)\|^2 \le \frac{k + 2D +2 }{k(2d-k)}
\]
\end{lemma}
Note that the lemma implies that there are $D$ vectors $\x_1,\ldots,\x_D\in \left\{\pm 1 \right\}^{d}$ for which $\Psi_k(\x_1),\ldots,\Psi_k(\x_D)$ are $\left( \frac{k + 2D}{k(2d-k)} \right)$-almost-orthonormal. In particular, for any constant $b>0$, there are $D=\Omega\left(\frac{dk}{\log(d)}\right)$ vecotrs $\x_1,\ldots,\x_D\in \left\{\pm 1 \right\}^{d}$ for which $\Psi_k(\x_1),\ldots,\Psi_k(\x_D)$ are $\left( \frac{b^2}{2\log(20D)} \right)$-almost-orthonormal.

\begin{lemma}\label{lem:Q_is_large}
For large enough $B>0$ and any choice of vectors $\x_1,\ldots,\x_D\in \left\{\pm 1 \right\}^{d}$, the Gaussian measure of $\tilde\cq_{d,k,B}(X)$ is $\ge \frac{8}{10}$
\end{lemma}

\begin{proof}
(of lemma \ref{lem:random_outer})
We first assume that $I_k := \sum_{i=1}^kE_{ii} \in V$. Here $E_{ii}$ is the matrix whose all elements are $0$ except the $ii$ entry which is $1$.
Let $E_1,\ldots, E_D$ be an orthonormal basis to $V$ such that $E_1 = \frac{1}{\sqrt{k}}I_k$. In particular, for all $i>1$, $\tr(E_i) = \sqrt{k}\inner{E_1,E_i} = 0$.  We note that for every $E\in H_{d,k}$ we have
\begin{eqnarray*}
\E\left[\inner{\Psi_k(X),E}^2\right] &=& \frac{1}{k(2d-k)}\E\left[\inner{XX^T,E}^2\right] 
\\
&=& \frac{\sum_{i\ne j} E_{ii}E_{jj} + 2E^2_{ij} + \sum_{i=1}^d E^2_{ii}}{k(2d-k)}
\\
&\le & \frac{\tr^2(E) + 2\|E\|_F^2}{k(2d-k)}
\end{eqnarray*}
Hence,
\begin{eqnarray*}
\E\|P_V\Psi_k(X)\|^2 &=& \sum_{i=1}^D \E\inner{\Psi_k(X),E_i}^2
\\
&\le & \frac{1}{k(2d-k)}\sum_{i=1}^D \tr^2(E_i) + 2\|E_i\|_F^2
\\
&=& \frac{\tr^2(E_1)  + 2D}{k(2d-k)} 
\\
&=& \frac{k + 2D}{k(2d-k)}
\end{eqnarray*}
In case that $I_k\notin V$, let $\tilde{V}$ be the linear span of $V\cup \{I_k\}$. By what we have shown and the fact that $\dim(\tilde{V}) = D+1$ we have
\[
\E\|P_V\Psi_k(X)\|^2 \le \E\|P_{\tilde{V}}\Psi_k(X)\|^2 = \frac{k + 2D +2 }{k(2d-k)}
\]
\end{proof}

\begin{proof}
(of lemma \ref{lem:Q_is_large}. Sketch)
Let $W\in H_{k,d}$	be a standard Gaussian, and let $\bu$ be the $k$'th normalized eigenvector of $W$ (with sign determined uniformly at random). It is not hard to see that the distribution of  $\bu$ is invariant to any diagonal $\pm 1$ matrix $U$. It follows that  given $(u_1^2,\ldots,u_d^2)$, $u_1x_1,\ldots,u_dx_d$ are independent random variables, and Hoefdding's bound implies that
$\Pr\left(|\inner{\bu,\x_i}| \ge \sqrt{2\ln(20dD)} \right)\le \frac{1}{10dD}$. Via  a union bound we conclude that the probability that $|\inner{\bu,\x_i}| \ge \sqrt{2\ln(20dD)}$ for some $i$ and normalized eigenvector $\bu$ is at most $\frac{1}{10}$. The lemma follows from that, together with the fact (e.g. Corollary 5.35 at \cite{vershynin2010introduction}) that with probability at least $1 - 2e^{-\frac{t^2}{2}}$, $\|W\|\le \sqrt{2d} + t$
\end{proof}

To prove lemma \ref{lem:shatter_by_large_measure} we will use Steele's generalization \cite{steele1978existence} of the VC dimension and Sauer-Shelah lemma

\begin{definition}
Let $\ch\subset Y^X$. A set $A\subset X$ is {\em shattered} if $\ch|_A = Y^A$. The {\em dimension} of $\ch$, denoted $\dim(\ch)$, is the maximal cardinality of a shattered set.
\end{definition}

\begin{lemma}[\cite{steele1978existence}] \label{lem:Steele} 
For any $\ch\subset Y^X$, $|\ch|\le \sum_{i=0}^{\dim(\ch)} \binom{|X|}{i}(|Y|-1)^{|X|-i}$
\end{lemma}

In the sequel we denote for vectors $\bv,\x$ in an inner product space $V$ and $a\in \reals$,
\[
h_{\bv,a}(\x):=
\begin{cases}
1 & \inner{\bv,\x} \ge a
\\
* & -a < \inner{\bv,\x} < a
\\
-1 & \inner{\bv,\x} \le -a
\end{cases}
\]

\begin{lemma}\label{lem:prob_to_hit_h}
Let $a>0$ be the scalar such that $\Pr_{X\sim\cn(0,1)}\left(X\in (-a,a)\right) = \frac{1}{3}$. 
Let $b > 0 $ small enough such that for any $\mu \in [-b,b]$ and $\sigma^2 \in [1-b^2,1]$
\[
\max\left\{ \Pr_{X\sim\cn(\mu,\sigma^2)}\left(X\le -a\right),  \Pr_{X\sim\cn(\mu,\sigma^2)}\left( X \ge a\right),\Pr_{X\sim\cn(\mu,\sigma^2)}\left(-a < X < a\right), \right\}  \le \frac{1}{3} + \epsilon_0
\]
Fix unit vectors $\x_1,\ldots,\x_D\in V$ such that for any $k\in [D]$, $\|P_{V_{k-1}}\x_k\|^2 \le \frac{b^2}{2\ln(20D)}$, where
$V_k = \mathrm{span}\{\x_i\}_{i=1}^k$. Fix also $h:[D]\to \{-1,1,*\}$ and a standard Gaussian $\bw\in V$. Then, 
\[
\Pr\left( \forall i\in [D], h_{\bw,a}(\x_i) = h(i)  \text{ and } \inner{  \bw, P_{V_{i-1}} \x_i }^2\le \left\|P_{V_{i-1}} \x_i\right\|^2 2\ln(20D)\right) \le \left(\frac{1}{3} + \epsilon_0\right)^{D}
\]
\end{lemma}
\begin{proof}
Let  $A_k$ be the event 
\[
A_k = \left\{\bw :  \forall i\in [k], h_{\bw,a}(\x_i) = h(i)  \text{ and } \inner{  \bw, P_{V_{i-1}} \x_i }^2\le \left\|P_{V_{i-1}} \x_i\right\|^2 2\ln(20D) \right\}
\]
Since $A_k\subset A_{k-1}$,
\[
\Pr(A_k) = \Pr(A_k|A_{k-1})\Pr(A_{k-1})
\]
Hence,
\[
\Pr\left(A_D\right) = \prod_{k=1}^D  \Pr(A_k|A_{k-1})
\]
It is therefore enough to show that $\Pr(A_k|A_{k-1}) \le \frac{1}{3} + \epsilon_0$.

To see this, write $\bw = \bw_1 + \bw_2 + \bw_3$ where $\bw_1,\bw_2,\bw_3$ are independent standard Gaussians on $V_{k-1}$, the orthogonal complement of $V_{k-1}$ in $V_k$ and $V_k^\perp$. Note that $\bw\in A_{k-1}$ if and only if $\bw_1\in A_{k-1}$.
It holds that given that $\bw_1\in A_{k-1}$, $\bw\in A_k$ only if 
\[
\inner{\bw,\x_k} =  \inner{\bw_1,\x_k} + \inner{\bw_2,\x_k} \in I_{h(k)}\text{ , where $I_{-1} = (-\infty,-a], I_1 = [a,\infty)$ and $I_* = (-a,a)$}
\]
Now, given $\bw_1\in A_{k-1}$, $\inner{\bw,\x_k} =  \inner{\bw_1,\x_k} + \inner{\bw_2,\x_k}$ is a Gaussian of  variance $\sigma^2 = 1 - \|P_{V_{k-1}}\x_k\|^2\ge 1-b^2$ and mean $\mu = \inner{ \bw_1, \x_k} = \inner{  P_{V_{k-1}} \bw,  \x_k } = \inner{  \bw, P_{V_{k-1}} \x_k }$ whose absolute value satisfies
\[
|\mu|\le \left\|P_{V_{k-1}} \x_k \right\| \sqrt{2\ln(20D)} \le b
\]
It therefore follows that the probability that $\inner{W,\x_k\x_k^T}  \in I_{h(k)}$ is bounded by $\frac{1}{3} + \epsilon_0$.
\end{proof}

\begin{proof}(of lemma \ref{lem:shatter_by_large_measure})
Let $a,b$ be as in lemma \ref{lem:prob_to_hit_h} with $\epsilon_0 = \frac{1}{9}$ and denote by $\mu$ the standard Gaussian measure on $V$.
Define 
\[
\tilde\cw = \cw\cap\{\bw :  \forall i\in [D], \inner{  \bw, P_{V_{i-1}} \x_i }^2\le \left\|P_{V_{i-1}} \x_i\right\|^2 2\ln(20D) \}
\]
Since $\Pr_{X\sim\cn(0,1)}\left(|X|\ge t\right)\le 2e^{\frac{t^2}{2}}$ we have 
\[
\mu(\tilde\cw) \ge \mu(\cw) - \mu\left(\cw :  \exists i\in [D], \inner{  \bw, P_{V_{i-1}} \x_i }^2 > \left\|P_{V_{i-1}} \x_i\right\|^2 2\ln(20D)\right) \ge \frac{7}{10}
\]
We will show that $\cw$ $a$-shatters  a set of size $\Omega\left(D\right)$. Let $\ch = \{h_{\bw,a} : \bw \in \tilde W \}$. For any $h:[D]\to \{-1,1,*\}$ define $\tilde{\cw}_h = \{\bw\in \tilde\cw : h_{\bw,a} = h\}$. By lemma \ref{lem:prob_to_hit_h}, $\mu(\tilde{\cw}_h) \le \left(\frac{4}{9}\right)^D$. On the other hand
\[
\sum_{h\in\ch} \mu(\tilde{\cw}_h) = \sum_{h:[D]\to \{-1,1,*\}}\mu(\tilde{\cw}_h) = \mu(\tilde\cw) \ge \frac{7}{10}
\]
It follows that
\[
|\ch| \ge \frac{7}{10}\left(\frac{9}{4}\right)^D
\]
On the other hand, by Steele's lemma \ref{lem:Steele}
\[
|\ch| \le 2^{D} \binom{D}{\le \dim(\ch)} \le 2^D\left(\frac{eD}{\dim(\ch)}\right)^{\dim(\ch)}
\]
Hence
\[
\left(\frac{eD}{\dim(\ch)}\right)^{\dim(\ch)} \ge \frac{7}{10}\left(\frac{9}{8}\right)^D
\]
It follows that
\[
\dim(\ch) = \Omega\left(D\right)
\]
\end{proof}

\subsection{Shattering with other Activations}

\begin{definition}
We say that an activation $\rho:\reals\to\reals$ is {\em nice}. 
If there is a constant $C>0$ and a distribution $\mu$ on $[-2,2]\times [-C,C]$ such that for any $x\in [-1,1]$ it holds that $\E_{(a,b)\sim\mu}b\rho(x-a) = x^2$
\end{definition}

\begin{lemma}
The ReLU activation $\rho(x) = \max(0,x)$ is nice. 
\end{lemma}
\begin{proof}
We first claim that if $f:\reals\to\reals$ is smooth and compactly supported then $f = f''*\sigma$.
Indeed,
\begin{eqnarray*}
(f''*\rho)(x) &=& \int_{-\infty}^\infty f''(t)\rho(x-t)dt	
\\
&=&  \int_{-\infty}^\infty f''(t)\int_{-\infty}^{x-t}  \rho'(\tau)d\tau dt	
\\
&=&  \int_{-\infty}^\infty \int_{-\infty}^{x-t} f''(t)  \rho'(\tau)d\tau dt	
\\
&=&  \int_{-\infty}^\infty \int_{-\infty}^{x-\tau} f''(t)  \rho'(\tau)dt d\tau 	
\\
&=&  \int_{-\infty}^\infty  \sigma'(\tau) \int_{-\infty}^{x-\tau} f''(t) dt d\tau 	
\\
&=&  \int_{0}^\infty  \int_{-\infty}^{x-\tau} f''(t) dt d\tau 	
\\
&=&  \int_{0}^\infty  f'(x-\tau)  d\tau 	
\\
&=&  \int_{-\infty}^x  f'(\tau)  d\tau 	
\\
&=&  f(x) 	
\end{eqnarray*}
%
%
Now, let $f:\reals\to\reals$ be a function that is smooth, coincides with $x^2$ on $[-1,1]$ and supported in $[-2,2]$. 
For any $x\in [-1,1]$ we have
\[
x^2 = f(x) =  \left(f''*\rho\right)(x) =  \int_{-\infty}^\infty \rho(x-a)f''(a)  da = \int_{-\infty}^\infty \left( \|f''\|_1 \sign(f(a))\right)\rho(x-a)\frac{\left| f''(a) \right|}{\|f''\|_1}  da 
\]
The lemma thus holds for the distribution $\mu$ of the random variable $\left(a, \left( \|f''\|_1 \sign(f(a))\right)\right)$ where $a$ is sampled according to the density function $\frac{\left| f''(a) \right|}{\|f''\|_1}$
\end{proof}

Theorem \ref{thm:nn_lower} now follows from the following theorem.
\begin{theorem}
Let $\rho$ be a nice activation. Then, for any $R\le \sqrt{d}$ we have that $\fat(\cn_{1,R}^{\rho}(d,R^2,1)) = \Omega\left(\frac{dR^2}{\log^2(d)}\right)$
\end{theorem}

\begin{proof} 

In the proof we will allow neurons to have bias terms. This can be standardly eliminated by adding constant dimensions to shattered  vectors.
Fix $k\le d$ and let $A = \left\{\x_1,\ldots,\x_D\right\}\subset \{\pm 1\}^d$ be $D = \Theta\left(\frac{dk}{\log(d)}\right)$ vectors that are nicely shattered by $\cq_{d,k,B}$ for the universal constant $B$ from theorem \ref{thm:quadratic_shatter}. We will show that $\x_1,\ldots,\x_D$ are $1$-shattered by $\cn_{O(1),O(R)}^{\rho}(d,O(k\log(d)),1)$. By simple scaling arguments it follows that $\cn_{1,R}^{\rho}(d,k\log(d),1)$ shatters a set of size $ \Theta\left(\frac{dk}{\log(d)}\right)$. Choosing $k = \frac{R^2}{\log(d)}$ will establish the theorem.

Fix $g:A\to\{\pm 1\}$ it is enough to show that there is $f\in \cn_{O(1),O(R)}^{\rho}(d,O(k\log(d)),1)$ such that
\begin{equation}\label{eq:need_to_show}
\forall \x\in A,\;\;f(\x)g(\x)\ge 1
\end{equation}
Since $A$ is nicely shattered by $\cq_{d,k,B}$ there are orthogonal unit vectors $\bu_1,\ldots,\bu_k$ and numbers $\lambda_1,\ldots,\lambda_k \in \left[-\frac{B}{\sqrt{k}},\frac{B}{\sqrt{k}}\right]$ such that
\begin{equation*}
\forall \x\in A,\;\;f(\x)q(\x)\ge 1\text{ for }q(\x) = \sum_{i=1}^k\lambda_i\inner{\bu_i,\x}^2
\end{equation*}
and
\begin{equation*}
\forall i\in [k]\text{ and }\x\in A,\;\; \left|\inner{\bu_i,\x}\right| \le \sqrt{2\ln(20dD)}
\end{equation*}
We will create a random network with $nk$ hidden neurons,  where $n$ will be determined later. 
Denote by $U\in M_{kd}$ the matrix whose $i$'s row is $\bu_i$, and let $L:=\sqrt{2\ln(20dD)}$. The hidden weight matrix (without the biases) will be
\[
\frac{1}{L}
\left.\left[\begin{tabular}{llll}
$U$\\
$U$\\
$\vdots$\\
$U$\\
\end{tabular}
\right]\right\}n\text{ times}
\]
To generate the biases and the output weights we will sample $nk$ independent pairs $\left\{(a_{i,j},b_{i,j})\right\}_{1\le i\le k,1\le j\le n}$ from the distribution $\mu$ on $[-2,2]\times [-C,C]$ that satisfies $\E_{(a,b)\sim\mu}b\rho(x-a) = x^2$ for any $x\in[-1,1]$. The bias of the $(i(k-1)+j)$'th neuron will be $-a_{i,j}$, and the corresponding output weight will be $\frac{2\lambda_ib_{i,j}L}{n}$.
The network will then calculate the function
\[
f_{\ba,\bb}(\x) = \sum_{j=1}^n\sum_{i=1}^k \frac{2\lambda_ib_{i,j}L}{n}\rho\left(\frac{\inner{\bu_i,x}}{L} - a_{i,j}\right)
\]
Now, we have that for any $\x\in A$, $\E f_{\ba,\bb}(\x) = 2g(\x)$. Likewise, $f_{\ba,\bb}(\x)$ is a sum of $nk$ independent random variables,  bounded by $O\left(\frac{\log(d)}{n\sqrt{k}}\right)$. Using Hoeffding's bound and union bound, we can  choose $n = O(\log(d))$ so that with positive probability $\forall \x\in A, \; \left| f_{\ba,\bb}(\x) - 2q(\x) \right|<1$, implying the \eqref{eq:need_to_show} holds.
Finally, the spectral norm of the hidden weight matrix is
\[
\sqrt{\frac{n}{2\ln(20dD)}} = O(1)
\]
Hence, since the rank is at most $kn$, the Frobenius norm is
\[
O\left(\sqrt{kn}\right) = O\left(\sqrt{k\log(d)}\right)
\]
As for the output weights, the squared norm is 
\[
\frac{L^2}{n^2}\sum_{j=1}^n\sum_{i=1}^k\lambda_i^2b_{i_j}^2 \le  \frac{L^2}{n^2}\sum_{j=1}^n\sum_{i=1}^k \frac{O(1)}{k} = O(1)
\]
This implies that $f_{\ba,\bb}\in  \cn_{O(1),O(R)}^{\rho}(d,O(k\log(d)),1)$

\end{proof}

\section{Future Work}
As we elaborate next our work leaves many open directions for further research.
First, we used ADL in order to analyze the sample complexity of fully connected neural networks.
We believe however that our approach is quite flexible and can be used to analyze the sample complexity of many other classes of functions. Natural candidates are convolutional and residual networks, as well as magnitude bounds in terms of of norms other than the spectral and Euclidean norm. We also believe that ADL can be useful beyond supervised learning, and can be used to analyze the sample complexity of sub-space learning (such as PCA and dictionary learning), clustering, and more.
In even more generality, it is interesting to explore the scope ADL in analyzing sample complexity.
Is ADL a ``complete" framework? That is, does learnability implies low ADL? 

Second, our current analysis leaves much to be desired. There are many poly-log factors in our bounds, the activation is required to be strongly bounded (and in particular, the ReLU activation is not captured), 
the loss function should be bounded, it is not clear whether the use of seeds in necessary, etc. Getting over these shortcomings is left for future work, which will hopefully lead to a cleaner theory.

Lastly, we note that our lower bound, theorem \ref{thm:nn_lower}, requires that $R\le \sqrt{d}$. We believe that this requirement in unnecessary, and the lower bound should hold for much larger  $R$'s.

\subsection*{Acknowledgements}
The authors acknowledge Kunal Talwar for many discussions in early stages of this work. The authors also acknowledge Haim Kaplan, Aryeh Kontorovich, and Yoram Singer for many useful comments.

\bibliography{bib}

\begin{thebibliography}{14}
\providecommand{\natexlab}[1]{#1}
\providecommand{\url}[1]{\texttt{#1}}
\expandafter\ifx\csname urlstyle\endcsname\relax
  \providecommand{\doi}[1]{doi: #1}\else
  \providecommand{\doi}{doi: \begingroup \urlstyle{rm}\Url}\fi

\bibitem[Anthony and Bartlet(1999)]{AnthonyBa99}
Martin Anthony and Peter Bartlet.
\newblock \emph{Neural Network Learning: Theoretical Foundations}.
\newblock Cambridge University Press, 1999.

\bibitem[Arora et~al.(2018)Arora, Ge, Neyshabur, and Zhang]{arora2018stronger}
Sanjeev Arora, Rong Ge, Behnam Neyshabur, and Yi~Zhang.
\newblock Stronger generalization bounds for deep nets via a compression
  approach.
\newblock In \emph{ICML}, 2018.

\bibitem[Bartlett and Mendelson(2002)]{BartlettMe02}
P.~L. Bartlett and S.~Mendelson.
\newblock Rademacher and {G}aussian complexities: {R}isk bounds and structural
  results.
\newblock \emph{Journal of Machine Learning Research}, 3:\penalty0 463--482,
  2002.

\bibitem[Bartlett et~al.(2017)Bartlett, Foster, and
  Telgarsky]{bartlett2017spectrally}
Peter~L Bartlett, Dylan~J Foster, and Matus~J Telgarsky.
\newblock Spectrally-normalized margin bounds for neural networks.
\newblock In \emph{Advances in Neural Information Processing Systems}, pages
  6240--6249, 2017.

\bibitem[Golowich et~al.(2018)Golowich, Rakhlin, and Shamir]{golowich2017size}
Noah Golowich, Alexander Rakhlin, and Ohad Shamir.
\newblock Size-independent sample complexity of neural networks.
\newblock In \emph{COLT}, 2018.

\bibitem[Nagarajan and Kolter(2019)]{nagarajan2019generalization}
Vaishnavh Nagarajan and J~Zico Kolter.
\newblock Generalization in deep networks: The role of distance from
  initialization.
\newblock \emph{arXiv preprint arXiv:1901.01672}, 2019.

\bibitem[Neyshabur et~al.(2015)Neyshabur, Tomioka, and
  Srebro]{neyshabur2015norm}
Behnam Neyshabur, Ryota Tomioka, and Nathan Srebro.
\newblock Norm-based capacity control in neural networks.
\newblock In \emph{Conference on Learning Theory}, pages 1376--1401, 2015.

\bibitem[Neyshabur et~al.(2018)Neyshabur, Bhojanapalli, and
  Srebro]{neyshabur2017pac}
Behnam Neyshabur, Srinadh Bhojanapalli, and Nathan Srebro.
\newblock A pac-bayesian approach to spectrally-normalized margin bounds for
  neural networks.
\newblock In \emph{ICLR}, 2018.

\bibitem[Neyshabur et~al.(2019)Neyshabur, Li, Bhojanapalli, LeCun, and
  Srebro]{neyshabur2018role}
Behnam Neyshabur, Zhiyuan Li, Srinadh Bhojanapalli, Yann LeCun, and Nathan
  Srebro.
\newblock The role of over-parametrization in generalization of neural
  networks.
\newblock In \emph{ICLR}, 2019.

\bibitem[Schapire et~al.(1997)Schapire, Freund, Bartlett, and
  Lee]{SchapireFrBaLe97}
R.E. Schapire, Y.~Freund, P.~Bartlett, and W.S. Lee.
\newblock Boosting the margin: A new explanation for the effectiveness of
  voting methods.
\newblock In \emph{Machine Learning: Proceedings of the Fourteenth
  International Conference}, pages 322--330, 1997.
\newblock To appear, {\em The Annals of Statistics}.

\bibitem[Shalev-Shwartz and Ben-David(2014)]{shalev2014understanding}
Shai Shalev-Shwartz and Shai Ben-David.
\newblock \emph{Understanding machine learning: From theory to algorithms}.
\newblock Cambridge university press, 2014.

\bibitem[Steele(1978)]{steele1978existence}
J~Michael Steele.
\newblock Existence of submatrices with all possible columns.
\newblock \emph{Journal of Combinatorial Theory, Series A}, 24\penalty0
  (1):\penalty0 84--88, 1978.

\bibitem[Sutskever et~al.(2013)Sutskever, Martens, Dahl, and
  Hinton]{sutskever2013importance}
Ilya Sutskever, James Martens, George Dahl, and Geoffrey Hinton.
\newblock On the importance of initialization and momentum in deep learning.
\newblock In \emph{International conference on machine learning}, pages
  1139--1147, 2013.

\bibitem[Vershynin(2010)]{vershynin2010introduction}
Roman Vershynin.
\newblock Introduction to the non-asymptotic analysis of random matrices.
\newblock \emph{arXiv preprint arXiv:1011.3027}, 2010.

\end{thebibliography}

\end{document}